\documentclass[12pt,journal,draftcls,a4paper,twoside,onecolumn]{IEEEtran}

\usepackage{setspace}
\usepackage{dsfont}
\usepackage{amsmath,amsthm}
\usepackage{amssymb}
\usepackage{amsfonts}
\usepackage{graphicx}
\usepackage[all,poly]{xy}
\usepackage{multirow}
\usepackage{epstopdf}
\usepackage{color}
\usepackage[noadjust]{cite}
\usepackage{subfigure}
\usepackage{mathrsfs}
\usepackage{algorithmic}
\usepackage[linesnumbered,lined,ruled]{algorithm2e}
\usepackage{cancel}
\usepackage{epsfig}
\usepackage{stfloats}
\usepackage{placeins}
\usepackage{float}
\usepackage{url}

\newfloat{bottomtwocolumnequation}{b}{fff}
\floatname{bottomtwocolumnequation}{Equation}

\newtheorem{theorem}{Theorem}

\newtheorem{lemma}{Lemma}
\newtheorem{assumption}{Assumption}

\newcommand{\CirEig}{\bfD}
\newcommand{\bs}{\boldsymbol}

\newcommand{\NoiCovMat}{\bs{\Lambda}}
\newcommand{\tcpp}[1]{\tcp{\emph{\small{#1}}}}
\newcommand{\undS}{\underline{\bfS}}
\newcommand{\EigSq}{\underline{\CirEig}}
\DeclareMathOperator{\Tr}{tr}

\def\bff{{\mathbf{f}}}

\def\bfu{{\mathbf{u}}}

\def\bfA{{\mathbf{A}}}
\def\bfB{{\mathbf{B}}}
\def\bfC{{\mathbf{C}}}
\def\bfD{{\mathbf{D}}}
\def\bfE{{\mathbf{E}}}
\def\bfF{{\mathbf{F}}}

\def\bfH{{\mathbf{H}}}

\def\bfJ{{\mathbf{J}}}

\def\bfL{{\mathbf{L}}}
\def\bfM{{\mathbf{M}}}
\def\bfN{{\mathbf{N}}}

\def\bfP{{\mathbf{P}}}
\def\bfQ{{\mathbf{Q}}}
\def\bfR{{\mathbf{R}}}
\def\bfS{{\mathbf{S}}}

\def\bfU{{\mathbf{U}}}
\def\bfV{{\mathbf{V}}}
\def\bfW{{\mathbf{W}}}
\def\bfX{{\mathbf{X}}}
\def\bfY{{\mathbf{Y}}}

\def\bsx{{\boldsymbol{x}}}

\def\wtm{\widetilde{m}}

\newcommand{\apriori}{\emph{a priori }}

\newcommand{\ALLobs}{\boldsymbol{\Psi}}

\newcommand{\MATima}{\bfX}

\newcommand{\noisevar}[1]{{s^2_{#1}}}

\newcommand{\nbbandima}{m_{\lambda}}

\newcommand{\Covsub}{\boldsymbol{\Sigma}}

\newcommand{\hypervect}{\boldsymbol{\Phi}}

\newcommand{\argmax}{\mathrm{arg}\max}
\newcommand{\argmin}{\mathrm{arg}\min}

\newcommand{\Vzeros}[1]{\boldsymbol{0}_{#1}}
\newcommand{\Id}[1]{\textbf{I}_{#1}}

\newcounter{algo}
\renewcommand{\thealgo}{\arabic{algo}} 

\graphicspath{{figures/}}

\begin{document}
\title{Fast Fusion of Multi-Band Images \\ Based on Solving a Sylvester Equation}

\author{\IEEEauthorblockN{Qi Wei, \IEEEmembership{Student Member,~IEEE},
Nicolas Dobigeon, \IEEEmembership{Senior Member,~IEEE},\\ and Jean-Yves
Tourneret, \IEEEmembership{Senior Member,~IEEE}}
\thanks{Part of this work has been supported by the Chinese Scholarship Council,
the Hypanema ANR Project n$^\circ$ANR-12-BS03-003 and by ANR-11-LABX-0040-CIMI within the program
ANR-11-IDEX-0002-02 within the thematic trimester on image processing. }
\thanks{Qi Wei, Nicolas Dobigeon and Jean-Yves
Tourneret are with University of Toulouse, IRIT/INP-ENSEEIHT, 2 rue
Camichel, BP 7122, 31071 Toulouse cedex 7, France (e-mail: \{qi.wei,
nicolas.dobigeon, jean-yves.tourneret\}@enseeiht.fr).} }

\maketitle
\begin{abstract}
This paper proposes a fast multi-band image fusion algorithm, which combines
a high-spatial low-spectral resolution image and a low-spatial high-spectral
resolution image. The well admitted forward model is explored to form the likelihoods
of the observations. Maximizing the likelihoods leads to solving a Sylvester equation.
By exploiting the properties of the circulant and downsampling
matrices associated with the fusion problem, a closed-form solution for
the corresponding Sylvester equation is obtained explicitly,
getting rid of any iterative update step. Coupled with the alternating direction method
of multipliers and the block coordinate descent method, the proposed algorithm can be
easily generalized to incorporate prior information for the fusion problem, allowing a
Bayesian estimator. Simulation results
show that the proposed algorithm achieves the same performance as existing
algorithms with the advantage of significantly decreasing the computational complexity
of these algorithms.
\end{abstract}


\section{Introduction}
\label{sec:intro}
\subsection{Background}
\label{subsec:background}
In general, a multi-band image can be represented as a
three-dimensional data cube indexed by three exploratory variables $(x,y,\lambda)$, where $x$ and $y$
are the two spatial dimensions of the scene, and $\lambda$ is the spectral dimension (covering a
range of wavelengths). Typical examples of multi-band images include hyperspectral
(HS) images \cite{Landgrebe2002}, multi-spectral (MS) images \cite{Navulur2006}, integral
field spectrographs \cite{Bacon2001}, magnetic resonance spectroscopy images
etc. However, multi-band imaging generally suffers from the limited spatial resolution of the data
acquisition devices, mainly due to an unsurpassable tradeoff between spatial and spectral sensitivities \cite{Chang2007}. For example, HS images benefit from excellent spectroscopic properties with hundreds of bands but are limited by their relatively low spatial resolution compared with MS and panchromatic (PAN) images (which are acquired in much fewer bands). As a consequence, reconstructing a high-spatial
and high-spectral multi-band image from two degraded and complementary observed images is a challenging but crucial issue that has been addressed in various scenarios. In particular, fusing a high-spatial low-spectral resolution image and a
low-spatial high-spectral image is an archetypal instance of multi-band image reconstruction, such as pansharpening (MS+PAN) \cite{Aiazzi2012} or hyperspectral pansharpening (HS+PAN) \cite{Loncan2015}.
Generally, the linear degradations applied to the observed images with respect
to (w.r.t.) the target high-spatial and high-spectral image reduce to
spatial and spectral transformations. Thus, the multi-band image fusion
problem can be interpreted as restoring a three dimensional data-cube from two
degraded data-cubes. A more precise description of the problem formulation is provided in the following paragraph.


\subsection{Problem Statement}
\label{subsec:prob_stat}
To better distinguish spectral and spatial degradations, the pixels of the target
multi-band image, which is of high-spatial and high-spectral
resolution, can be rearranged to build an $\nbbandima \times n$
matrix $\bfX$, where $\nbbandima$ is the number of spectral bands and
$n=n_r \times n_c$ is the number of pixels in each band ($n_r$ and $n_c$ represents the number of rows and columns respectively). In other words, each column of the matrix $\bfX$ consists of a $\nbbandima$-valued pixel and each row gathers all the pixel values in a given spectral band.
Based on this pixel ordering, any linear operation applied on the left (resp. right) side of $\bfX$ describes a spectral (resp. spatial) degradation.

In this work, we assume that two complementary images of high-spectral or high-spatial resolutions, respectively, are available to reconstruct the target high-spectral and high-spatial resolution target image. These images result from linear spectral and spatial degradations of the full resolution image $\bfX$, according to the well-admitted model
\begin{equation}
\begin{array}{ll}
\label{eq:obs_general}
\bfY_{\mathrm{L}} =  {\bfL} \MATima + \bfN_{\mathrm{L}}\\
\bfY_{\mathrm{R}} =  \MATima \bfR + \bfN_{\mathrm{R}}
\end{array}
\end{equation}
where
\begin{itemize}
\item $\MATima = \left[\bsx_1,\ldots,\bsx_{n}\right] \in \mathbb{R}^{\nbbandima \times n}$ is
	the full resolution target image,
\item ${\bfY}_{\mathrm{L}}\in\mathbb{R}^{n_{\lambda} \times n}$ and ${\bfY}_{\mathrm{R}} \in \mathbb{R}^{\nbbandima \times m}$
are the observed spectrally degraded and spatially degraded images,
\item $m=m_r\times m_c$ is the number of pixels of the high-spectral resolution image,
\item $n_{\lambda}$ is the number of bands of the high-spatial resolution image,
\item $\bfN_{\mathrm{L}}$ and $\bfN_{\mathrm{R}}$ are additive terms that include
both modeling errors and sensor noises.
\end{itemize}

The noise matrices are assumed to be
distributed according to the following matrix normal distributions\footnote{The probability density function $p(\bfX | \bfM, \bs{\Sigma}_r, \bs{\Sigma}_c)$
of a matrix normal distribution $\mathcal{MN}_{r,c}(\mathbf{M}, \bs{\Sigma}_r, \bs{\Sigma}_c)$
is defined by
\begin{equation*}
\begin{split}
p(\bfX | \bfM, \bs{\Sigma}_r, \bs{\Sigma}_c)=
\frac{\exp\left( -\frac{1}{2}\mathrm{tr}\left[ \bs{\Sigma}_c^{-1} (\mathbf{X} - \mathbf{M})^{T} \bs{\Sigma}_r^{-1} (\mathbf{X} - \mathbf{M}) \right] \right)}{(2\pi)^{rc/2} |\bs{\Sigma}_c|^{r/2} |\bs{\Sigma}_r|^{c/2}}
\end{split}
\end{equation*}
where $\bfM \in \mathbb{R}^{r \times c}$ is the mean matrix, $\bs{\Sigma}_r \in \mathbb{R}^{r \times r}$ is
the row covariance matrix and $\bs{\Sigma}_c \in \mathbb{R}^{c \times c}$ is the column covariance matrix.}
\begin{equation*}
    \begin{array}{ll}
      \bfN_{\mathrm{L}}  \sim \mathcal{MN}_{\nbbandima,m}(\Vzeros{\nbbandima,m},\NoiCovMat_{\mathrm{L}} ,\Id{m} ) \\
      \bfN_{\mathrm{R}}  \sim \mathcal{MN}_{n_{\lambda},n}(\Vzeros{n_{\lambda},n}, \NoiCovMat_{\mathrm{R}}, \Id{n}).
    \end{array}
\end{equation*}
Note that no particular structure is assumed for the row covariance matrices $\NoiCovMat_{\mathrm{L}}$ and
$\NoiCovMat_{\mathrm{R}}$ except that they are both positive definite, which allows for considering spectrally colored noises. Conversely, the column
covariance matrices are assumed to be the identity matrix to reflect the fact that the noise is pixel-independent.

In most practical scenarios, the spectral degradation ${\bfL} \in\mathbb{R}^{n_{\lambda} \times \nbbandima}$
only depends on the spectral response of the sensor, which can be \emph{a priori} known or
estimated by cross-calibration \cite{Yokoya2013cross}.
The spatial degradation $\bfR$ includes warp, translation, blurring, decimation, etc. As the warp and translation
can be attributed to the image co-registration problem and mitigated by precorrection, only blurring and decimation degradations,
denoted $\bfB$ and $\bfS$ are considered in this work. If the spatial blurring is assumed to be
space-invariant, $\bfB \in \mathbb{R}^{n \times n}$ owns the specific property of being a cyclic convolution operator acting on the
bands. The matrix ${\bf S}\in\mathbb{R}^{n\times m}$ is a $d=d_r \times d_c$ uniform downsampling operator, which
has $m=n/d$ ones on the block diagonal and zeros elsewhere, and such that ${\bfS}^T \bfS=\Id{m}$. Note that
multiplying by $\bfS^T$ represents zero-interpolation to increase the number of pixels from $m$ to $n$.
Therefore, assuming $\bfR$ can be decomposed as $\bfR= {\bf BS} \in \mathbb{R}^{n \times m}$, the
fusion model \eqref{eq:obs_general} can be rewritten as
\begin{equation}
\begin{array}{ll}
\bfY_{\mathrm{L}} =  {\bfL} \MATima + \bfN_{\mathrm{L}}\\
\bfY_{\mathrm{R}} =  \MATima {\bf BS} + \bfN_{\mathrm{R}}
\end{array}
\label{eq:obs_specific}
\end{equation}
where all matrix dimensions and their respective relations are
summarized in Table \ref{tb:size}.

\begin{table}[h!]
\centering \caption{Notations}
\small
\setlength{\tabcolsep}{0.5mm}
\begin{tabular}{c|c|c}
\hline
\textbf{Notation} & \textbf{Definition} & \textbf{Relation} \\
\hline
$m_r$      & row number of $\bfY_{\mathrm{R}}$    & $m_r=n_r / d_r$\\
$m_c$      & column number of $\bfY_{\mathrm{R}}$  & $m_c=n_c / d_c$\\
$m $	   & number of pixels in each band of $\bfY_{\mathrm{R}}$ & $m=m_r \times m_c$\\
$n_r$      & row number of $\bfY_{\mathrm{L}}$    &   $n_r=m_r \times d_r$\\
$n_c$      & column number of $\bfY_{\mathrm{L}}$ & $n_c=m_c \times d_c$\\
$n $	   & number of pixels in each band of $\bfY_{\mathrm{L}}$ & $n=n_r \times n_c$\\
$d_r$      & decimation factor in row     &$d_r=n_r / m_r$\\
$d_c$      & decimation factor in column  &$d_c=n_c / m_c$ \\
$d$        & decimation factor   & $m=d_r \times d_c$\\
\hline
\end{tabular}
\label{tb:size}
\end{table}

This matrix equation \eqref{eq:obs_general} has been widely advocated in
the pansharpening and HS pansharpening problems, which consist of
fusing a PAN image with an MS or an HS image \cite{Amro2011survey,Gonzalez2004fusion,Loncan2015}. Similarly, most of the techniques developed to fuse MS and HS images also rely on
a similar linear model \cite{Hardie2004,Molina1999,Molina2008,Zhang2012,Yokoya2012coupled,Wei2014Bayesian,Wei2015TGRS}. From an applicative point of view,
this problem is also important as motivated by recent national programs,
e.g., the Japanese next-generation space-borne HS image suite (HISUI), which fuses co-registered MS and HS images acquired over the same scene under the same conditions \cite{Yokoya2013}.

To summarize, the problem of fusing high-spectral and high-spatial resolution images
can be formulated as estimating the unknown matrix $\bfX$ from \eqref{eq:obs_specific}.
There are two main statistical estimation methods that can be used to solve this problem. These methods are based on maximum likelihood (ML) or on Bayesian inference. ML estimation is purely
data-driven while Bayesian estimation relies on prior information, which can be
regarded as a regularization (or a penalization) for the fusion problem.
Various priors have been already advocated to regularize the multi-band image fusion problem, such as Gaussian priors \cite{Wei2013,Wei2015whispers},
sparse representations \cite{Wei2015TGRS} or total variation (TV) \cite{Simoes2015} priors.
The choice of the prior usually depends on the information resulting from previous experiments
or from a subjective view of constraints
affecting the unknown model parameters \cite{Robert2007,Gelman2013}.

Computing the ML or the Bayesian estimators (whatever the form chosen for the prior) is a challenging task, mainly due to the large size of $\bfX$ and to the
presence of the downsampling operator $\bfS$, which prevents any direct use
of the Fourier transform to diagonalize the blurring operator $\bfB$. To overcome this difficulty, several computational strategies have been designed to approximate the estimators. Based on a Gaussian prior modeling, a Markov chain Monte Carlo (MCMC) algorithm has been implemented in  \cite{Wei2013} to generate a collection
of samples asymptotically distributed according to the posterior distribution of $\bfX$. The Bayesian estimators of $\bfX$ can then be approximated using these samples. Despite this formal appeal,
MCMC-based methods have the major drawback of being computationally expensive, which prevents their effective use when processing images of large size. Relying on exactly the same prior model, the strategy developed in \cite{Wei2015whispers} exploits an alternating direction method of multipliers (ADMM) embedded in a block coordinate descent method (BCD) to compute the maximum a posterior (MAP) estimator of $\bfX$. This optimization strategy allows the numerical complexity to be greatly decreased when compared to its MCMC counterpart. Based on a prior built from a sparse representation, the fusion problem is solved in \cite{Simoes2015,Wei2015TGRS} with the split augmented Lagrangian shrinkage algorithm (SALSA) \cite{Afonso2011}, which is an
instance of ADMM.

In this paper, contrary to the algorithms described above, a much more efficient method is proposed to solve explicitly
an underlying Sylvester equation (SE) associated with the fusion problem derived from \eqref{eq:obs_specific}. This solution can be implemented \emph{per se} to compute the ML estimator in a computationally efficient manner. The proposed solution has also the great advantage of being easily generalizable within a Bayesian framework when considering various priors. The MAP estimators associated with a Gaussian prior similar to \cite{Wei2013,Wei2015whispers} can be directly computed thanks to the proposed strategy. When handling more complex priors such as \cite{Simoes2015,Wei2015TGRS}, the SE solution can be conveniently embedded within a conventional ADMM or a BCD algorithm.

\subsection{Paper organization}
The remaining of this paper is organized as follows.
Section II studies the optimization problem to be addressed
in absence of any regularization, i.e., in an ML framework.
The proposed fast fusion method is presented in Section III
and generalized to Bayesian estimators associated with various priors in
Section IV. Section V presents experimental results assessing the accuracy and
the numerical efficiency of the proposed fusion method. Conclusions are
finally reported in Section \ref{sec:concls}.


\section{Problem Formulation}
\label{sec:ObsModel}
Using the statistical properties of the noise matrices $\bfN_{\mathrm{L}}$ and $\bfN_{\mathrm{R}}$,
$\bfY_{\mathrm{L}}$ and $\bfY_{\mathrm{R}}$ have matrix Gaussian distributions, i.e.,
\begin{equation}
\begin{array}{ll}
\label{eq:Likelihood}
p\left(\bfY_{\mathrm{L}}|\bfX\right) = \mathcal{MN}_{n_{\lambda},n}({\bfL}\bfX,\NoiCovMat_{\mathrm{L}}, \Id{n})\\
p\left(\bfY_{\mathrm{R}}|\bfX\right) = \mathcal{MN}_{\nbbandima,m}(\bfX \bfB \bfS, \NoiCovMat_{\mathrm{R}}, \Id{m}).
\end{array}
\end{equation}

As the collected measurements $\bfY_{\mathrm{L}}$ and $\bfY_{\mathrm{R}}$ have been
acquired by different (possibly heterogeneous) sensors, the noise matrices $\bfN_{\mathrm{L}}$
and $\bfN_{\mathrm{R}}$ are sensor-dependent and can be generally assumed to be statistically independent.
Therefore, $\bfY_{\mathrm{L}}$ and $\bfY_{\mathrm{R}}$ are independent conditionally upon
the unobserved scene $\bfX$. As a consequence, the joint likelihood function of the observed
data is
\begin{equation}
\label{eq:Likelihood_joint}
p\left(\bfY_{\mathrm{L}},\bfY_{\mathrm{R}}|\bfX\right)=p\left(\bfY_{\mathrm{L}}|\bfX\right)p\left(\bfY_{\mathrm{R}}|\bfX\right).
\end{equation}
Since adjacent HS bands are known to be highly correlated, the HS vector $\bsx_i$
usually lives in a subspace whose dimension is much smaller than the number
of bands $\nbbandima$ \cite{Chang1998,Bioucas2008}, i.e., $\bf X = HU$
where $\bfH$ is an orthogonal matrix such that ${\bfH}^T {\bfH}=\Id{\wtm_{\lambda}}$ and
$\bfU \in \mathbb{R}^{\wtm_{\lambda} \times n}$ is the projection of $\bfX$ onto the subspace
spanned by the columns of $\bfH \in \mathbb{R}^{\nbbandima  \times \wtm_{\lambda}}$.

Defining $\ALLobs=\left\{\bfY_{\mathrm{L}}, \bfY_{\mathrm{R}}\right\}$ as the set of the observed images,
the negative logarithm of the likelihood is
\begin{equation*}
\begin{array}{ll}
&-\log p\left(\ALLobs|\bfU\right)  =  -\log p\left(\bfY_{\mathrm{L}}|\bfU\right) -\textrm{log } p\left(\bfY_{\mathrm{R}}|\bfU\right) +C\\
&\quad \quad =\frac{1}{2}\Tr \left( {\left(\bfY_{\mathrm{R}}- \bf{HUBS}\right)^T\NoiCovMat_{\mathrm{R}}^{-1}\left(\bfY_{\mathrm{R}}- \bf{HUBS}\right)} \right) + \\
&\quad \quad \frac{1}{2} \Tr \left(\left(\bfY_{\mathrm{L}} -{{\bfL}\bf HU}\right)^T \NoiCovMat_{\mathrm{L}}^{-1}\left(\bfY_{\mathrm{L}} -{{\bfL}\bf HU}\right)\right)+C
\end{array}
\end{equation*}
where $C$ is a constant. Thus, calculating the ML estimator of $\bfU$ from the observed images $\ALLobs$,
i.e., maximizing the likelihood can be achieved
by solving the following problem
\begin{equation}
\label{eq:neglog_likeli}
\argmin\limits_{\bfU} L(\bfU)
\end{equation}
where
\begin{equation*}
\begin{split}
L(\bfU) = \Tr \left( {\left(\bfY_{\mathrm{R}}- \bf{HUBS}\right)^T\NoiCovMat_{\mathrm{R}}^{-1}\left(\bfY_{\mathrm{R}}- \bf{HUBS}\right)} \right) + \\
\Tr \left(\left(\bfY_{\mathrm{L}} -{{\bfL}\bf HU}\right)^T \NoiCovMat_{\mathrm{L}}^{-1}\left(\bfY_{\mathrm{L}} -{{\bfL}\bf HU}\right)\right).
\end{split}
\end{equation*}

Note that it is also obvious to formulate the optimization problem \eqref{eq:neglog_likeli}
from the linear model \eqref{eq:obs_specific} directly in the least-squares (LS) sense \cite{Lawson1974}.
However, specifying the distributions of the noises $\NoiCovMat_{\mathrm{L}}$ and $\NoiCovMat_{\mathrm{R}}$
allows us to consider the case of colored noises (band-dependent) more easily by introducing
the covariance matrices $\NoiCovMat_{\mathrm{R}}$ and $\NoiCovMat_{\mathrm{L}}$,
leading to the weighting LS problem \eqref{eq:neglog_likeli}.

In this paper, we prove that the minimization of \eqref{eq:neglog_likeli} w.r.t. the target image $\bfU$ can be solved
analytically, without any iterative optimization scheme or Monte Carlo 
based method. The resulting closed-form solution to the optimization problem is presented in Section \ref{sec:FastFus}.
Furthermore, it is shown in Section \ref{sec:gen_Fast} that the proposed method can be easily generalized
to Bayesian fusion methods with appropriate prior distributions.

\section{Fast Fusion Scheme}
\label{sec:FastFus}
\subsection{Sylvester equation}
\label{subsec:opt_ima}
Minimizing \eqref{eq:neglog_likeli} w.r.t. $\bfU$ is equivalent
to force the derivative of $L(\bfU)$ to be zero, i.e., ${ \mathrm{d} L(\bfU)}/{ \mathrm{d} \bfU}  =0 $,
leading to the following matrix equation
\begin{equation}
\label{eq:sylvester}
\begin{array}{ll}
\bfH^H \NoiCovMat_{\mathrm{R}}^{-1} {\bf{HUBS}} \left(\bf BS \right)^H + \left({\left({\bfL}\bfH\right)}^H \NoiCovMat_{\textrm{L}}^{-1} {{\bfL}\bfH}\right) \bfU \\
= \bfH^H \NoiCovMat_{\mathrm{R}}^{-1} {\bfY}_\mathrm{R} \left(\bf BS \right)^H + {\left({\bfL}\bfH\right)}^H \NoiCovMat_{\textrm{L}}^{-1} {\bfY}_\mathrm{L}.\\
\end{array}
\end{equation}
As mentioned in Section \ref{subsec:prob_stat}, the difficulty for solving \eqref{eq:sylvester} results from
the high dimensionality of $\bfU$ and the presence of the downsampling matrix $\bfS$. In this work, we will show that
Eq. \eqref{eq:sylvester} can be solved analytically with some assumptions summarized below.

\begin{assumption}
\label{ass:3}
The blurring matrix $\bfB$ is a circulant matrix.
\end{assumption}
A consequence of this assumption is that $\bfB$ can be decomposed as $\bfB = \bfF \CirEig \bfF^H$ and
$\bfB^H = \bfF \CirEig^{\ast} \bfF^H$, where $\bfF \in \mathbb{R}^{n\times n}$
is the discrete Fourier transform (DFT) matrix ($\bfF\bfF^H=\bfF^H\bfF=\Id{n}$), $\bfD \in \mathbb{R}^{n\times n}$
is a diagonal matrix and $\ast$ represents the conjugate operator.

\begin{assumption}
\label{ass:2}
The decimation matrix $\bfS$ corresponds to downsampling the original signal and its conjugate transpose $\bfS^H$ interpolates the decimated signal with zeros.
\end{assumption}
A decimation matrix satisfies the property $\bfS^H\bfS=\Id{m}$. Moreover,
the matrix $\undS \triangleq
\bfS\bfS^H \in \mathbb{R}^{n\times n}$ is symmetric and idempotent, i.e., $\undS=\undS^H$ and $\undS \undS^H=\undS^2=\undS$.
For a practical implementation, multiplying an image by $\undS$ can be achieved by doing entry-wise multiplication with
an $n\times n$ mask matrix with ones in the sampled position and zeros elsewhere.

Note that the assumptions w.r.t. the blurring matrix $\bfB$ and the decimation matrix $\bfS$ have been widely used in
image processing application, like super-resolution \cite{Elad1997,Park2003SR}, fusion \cite{Hardie2004,Simoes2015}, etc.

After multiplying \eqref{eq:sylvester} on both sides
by $\left(\bfH^H \NoiCovMat_{\mathrm{R}}^{-1} \bfH\right)^{-1}$\footnote{The invertibility of the matrix
$\bfH^H \NoiCovMat_{\mathrm{R}}^{-1} \bfH$ is guaranteed since $\bfH$ has full column rank
and $\NoiCovMat_{\mathrm{R}}$ is positive definite.}
we obtain
\begin{equation}
\label{eq:zeroforce}
\bfC_1 \bfU +  \bfU \bfC_2  = \bfC_3
\end{equation}
where
\begin{equation*}
\begin{array}{ll}
\bfC_1=\left(\bfH^H \NoiCovMat_{\mathrm{R}}^{-1} \bfH\right)^{-1}\left({\left({\bfL}\bfH\right)}^H \NoiCovMat_{\textrm{L}}^{-1} {{\bfL}\bfH}\right)\\
\bfC_2={\bfB \undS \bfB}^H\\
\bfC_3=\left(\bfH^H \NoiCovMat_{\mathrm{R}}^{-1} \bfH\right)^{-1} \left(\bfH^H \NoiCovMat_{\mathrm{R}}^{-1} {\bfY}_\mathrm{R} \left(\bf BS \right)^H + {\left({\bfL}\bfH\right)}^H\NoiCovMat_{\textrm{L}}^{-1} {\bfY}_\mathrm{L} \right).
\end{array}
\end{equation*}
Eq. \eqref{eq:zeroforce} is a generalized Sylvester matrix equation \cite{Bartels1972}.
It is well known that a SE  has a unique solution if and only if an arbitrary sum of the eigenvalues
of $\bfC_1$ and $\bfC_2$ is not equal to zero \cite{Bartels1972}.

\subsection{Existence of a solution}
In this section, we study the eigenvalues of $\bfC_1$ and $\bfC_2$ to check if \eqref{eq:zeroforce}
has a unique solution. As matrix $\bfC_2=\bfB \undS \bfB^H$ is positive semi-definite, its eigenvalues
include positive values and zeros \cite{Horn2012}. In order to study the eigenvalues of $\bfC_1$, Lemma \ref{lem:eig_sym}
is introduced below.

\begin{lemma}
\label{lem:eig_sym}
If matrix $\bfA_1 \in \mathbb{R}^{n\times n}$ is symmetric (resp. Hermitian) positive definite and matrix $\bfA_2 \in \mathbb{R}^{n\times n}$ is symmetric (resp. Hermitian)
positive semi-definite, the product $\bfA_1 \bfA_2$ is diagonalizable and all the eigenvalues of $\bfA_1 \bfA_2$ are non-negative.
\end{lemma}

\begin{proof}
See proof in Appendix \ref{app:eig_sym}.
\end{proof}

According to Lemma \ref{lem:eig_sym}, since the matrix $\bfC_1$ is the product of a symmetric positive definite matrix
$\left(\bfH^H \NoiCovMat_{\mathrm{R}}^{-1} \bfH\right)^{-1}$ and a symmetric semi-definite matrix
${\left({\bfL}\bfH\right)}^H \NoiCovMat_{\textrm{L}}^{-1} {{\bfL}\bfH}$, it is diagonalizable
and all its eigenvalues are non-negative. As a consequence, the eigen-decomposition of $\bfC_1$ can be
expressed as $\bfC_1=\bfQ{\bf{\Lambda}}_1\bfQ^{-1}$, where
${\bf{\Lambda}}_1=\textrm{diag}\left(\lambda_1,\cdots,\lambda_{\wtm_{\lambda}}\right)$ ($\textrm{diag}\left(\lambda_1,\cdots,\lambda_{\wtm_{\lambda}}\right)$ is a diagonal matrix whose elements are $\lambda_1,\cdots,\lambda_{\wtm_{\lambda}}$)
and $\lambda_i \geq 0$, $\forall i$. Therefore, as long as zero is not an eigenvalue of $\bfC_1$
(or equivalently $\bfC_1$ is invertible), any sum of eigenvalues of $\bfC_1$
and $\bfC_2$ is  different from zero (more accurately, this sum is $>0$), leading to the
existence of a unique solution of \eqref{eq:zeroforce}.

However, the invertibility
of $\bfC_1$ is not always guaranteed depending on the forms and dimensions of $\bfH$ and $\bfL$.
For example, if $n_{\lambda}<\wtm_{\lambda}$, meaning the number of MS bands is smaller than the subspace dimension, the matrix
${\left({\bfL}\bfH\right)}^H \NoiCovMat_{\textrm{L}}^{-1} {{\bfL}\bfH}$ is rank deficient and thus there is not a unique solution	
of \eqref{eq:zeroforce}. In cases when $\bfC_1$ is singular, a regularization or prior information is necessary to be introduced
to ensure \eqref{eq:zeroforce} has a unique solution. In this section, we focus on the case when $\bfC_1$ is non-singular. The generalization to Bayesian estimators based on specific priors already considered in the literature will be elaborated in Section \ref{sec:gen_Fast}.

\subsection{A classical algorithm for the Sylvester matrix equation}
A classical algorithm for obtaining a solution of the SE is the Bartels-Stewart algorithm \cite{Bartels1972}. This algorithm decomposes $\bfC_1$ and $\bfC_2$ into Schur forms using a QR algorithm and solves the resulting triangular
system via back-substitution. However, as the matrix $\bfC_2={{\bf{B}\undS \bfB}^H}$ is
very large for our application ($n \times n$, where $n$ is the number of image pixels), it is unfeasible to construct the matrix $\bfC_2$, let alone use the QR algorithm to compute its Schur form (which has the computational cost $\mathcal{O}(n^3)$ arithmetical operations).
The next section proposes an innovative strategy to obtain an analytically expression of the SE \eqref{eq:zeroforce} by exploiting the specific properties of the matrices $\bfC_1$ and $\bfC_2$ associated with the fusion problem.

\subsection{Proposed closed-form solution}
Using the decomposition $\bfC_1=\bfQ{\bf{\Lambda}}_C\bfQ^{-1}$ and multiplying both sides
of \eqref{eq:zeroforce} by $\bfQ^{-1}$ leads to
\begin{equation}
{\bf{\Lambda}}_C \bfQ^{-1} \bfU + \bfQ^{-1}\bfU \bfC_2  = \bfQ^{-1} \bfC_3.
\label{eq:sylv_1}
\end{equation}
Right multiplying \eqref{eq:sylv_1} by $\bf FD$ on both sides and using the definitions of matrices $\bfC_2$ and $\bf B$ yields
\begin{equation}
{\bf{\Lambda}}_1 \bfQ^{-1} {\bf{UFD}}+ \bfQ^{-1}{\bf{UFD}}\left( \bfF^H \undS{\bfF} \EigSq  \right) = \bfQ^{-1} \bfC_3 \bf FD
\label{eq:sylv_2}
\end{equation}
where $\EigSq =\left(\CirEig^{\ast}\right)\CirEig$ is a real diagonal matrix.
Note that ${\bf{UF}\CirEig}={\bf{UBF}} \in \mathbb{R}^{\wtm_{\lambda} \times n}$ can be interpreted as the
Fourier transform of the blurred target image, which is a complex matrix. Eq. \eqref{eq:sylv_2} can be regarded
as a SE w.r.t. $\bfQ^{-1} {\bf{UFD}}$, which has a simpler form compared to \eqref{eq:zeroforce}
as ${\bf{\Lambda}}_C$ is a diagonal matrix.

The next step in our analysis is to simplify the matrix $\bfF^H \undS{\bfF}\EigSq$ appearing on the left hand side of \eqref{eq:sylv_2}.
First, it is important to note that the matrix $\bfF^H \undS{\bfF}\EigSq$ has a specific structure since all its columns contain the same
blocks (see Eq. \eqref{eq:block_prod} in Appendix \ref{app:MMat}). Using this property, by multiplying left and right by specific matrices,
one obtains a block matrix whose nonzero blocks are located in its first (block) row (see \eqref{eq:MMat}).
More precisely, introduce the following matrix
\begin{equation}
{\bfP}=
\underbrace{\left[
\begin{array}{ccccc}
\Id{m} &\bs{0} &\cdots &\bs{0}\\
-\Id{m}& \Id{m}& \cdots &\bs{0}\\
\vdots&\vdots &\ddots &\vdots\\
-\Id{m}& \bs{0}& \cdots &\Id{m}
\end{array}
\right]}_{d}
\label{eq:Pmat}
\end{equation}
whose inverse\footnote{Note that left multiplying a matrix by $\bfP$ corresponds to subtracting the first row blocks from
all the other row blocks. Conversely, right multiplying by the matrix $\bfP^{-1}$ means
replacing the first (block) column by the sum of all the other (block) columns.} can be easily computed
\begin{equation*}
{\bfP}^{-1}=
\underbrace{\left[
\begin{array}{ccccc}
\Id{m} &\bs{0} &\cdots &\bs{0}\\
\Id{m}& \Id{m}& \cdots &\bs{0}\\
\vdots&\vdots &\ddots &\vdots\\
\Id{m}& \bs{0}& \cdots &\Id{m}
\end{array}
\right]}_{d}.
\end{equation*}
Right multiplying both sides of \eqref{eq:sylv_2} by ${\bfP}^{-1}$ leads to
\begin{equation}
{\bf{\Lambda}}_C \bar{\bfU} + \bar{\bfU}\bfM  = \bar{\bfC}_3
\label{eq:sylv_3}
\end{equation}
where $\bar{\bfU}={\bfQ^{-1}\bf{UFD}}{\bfP}^{-1}$, $\bfM={\bfP} \left( \bfF^H \undS{\bfF}\EigSq\right) {\bfP}^{-1}$
and $\bar{\bfC}_3={\bfQ^{-1}\bfC_3}{\bf FD} {\bfP}^{-1}$.
Eq. \eqref{eq:sylv_3} is a Sylvester matrix equation w.r.t. $\bar{\bfU}$ whose solution is significantly easier than for \eqref{eq:sylv_1} because the matrix $\bfM$ has a simple form outlined in the following lemma.
\begin{lemma}
The following equality holds
\begin{equation}
\bfM= {\bfP} \left( \bfF^H \undS{\bfF}\EigSq\right) {\bfP}^{-1}=
\frac{1}{d}\left[
\begin{array}{ccccc}
\sum\limits_{i=1}^{d}{\EigSq}_i &{\EigSq}_2 &\cdots &{\EigSq}_d\\
\bs{0}& \bs{0}& \cdots &\bs{0}\\
\vdots&\vdots &\ddots &\vdots\\
\bs{0}& \bs{0}& \cdots &\bs{0}
\end{array}
\right] \label{eq:MMat}
\end{equation}
where the matrix $\EigSq$ has been partitioned as follows
\begin{equation*}
\EigSq=\left[
\begin{array}{ccccc}
\EigSq_1 &\bs{0} &\cdots &\bs{0}\\
\bs{0}& \EigSq_2& \cdots &\bs{0}\\
\vdots&\vdots &\ddots &\vdots\\
\bs{0}& \bs{0}& \cdots &\EigSq_d
\end{array}
\right]
\end{equation*}
with $\EigSq_i$  $m \times m$ real diagonal matrices.
 \label{lem:MMat}
\end{lemma}

\begin{proof}
See Appendix \ref{app:MMat}.
\end{proof}

Finally, using the specific form of $\bfM$ provided by Lemma~\ref{lem:MMat}, the solution $\bar{\bfU}$ of the SE \eqref{eq:sylv_3} can be computed block-by-block as stated in the following theorem.
\begin{theorem}
\label{the:Ubar}
  Let $(\bar{\bfC}_3)_{l,j}$ denotes the $j$th block of the $l$th band of $\bar{\bfC}_3$ for any $l=1,\cdots,\wtm_{\lambda}$. Then, the solution  $\bar{\bfU}$ of the SE \eqref{eq:sylv_3} can be decomposed as
\begin{equation}
\bar{\bfU}=
\left[
\begin{array}{ccccc}
\bar{\bfu}_{1,1}& \bar{\bfu}_{1,2} &\cdots &\bar{\bfu}_{1,d}\\
\bar{\bfu}_{2,1}& \bar{\bfu}_{2,2} &\cdots &\bar{\bfu}_{2,d}\\
\vdots    & \vdots     &\ddots &\vdots\\
\bar{\bfu}_{\wtm_{\lambda},1}& \bar{\bfu}_{\wtm_{\lambda},2}& \cdots &\bar{\bfu}_{\wtm_{\lambda},d}
\end{array}
\right]
\label{eq:U_vector}
\end{equation}
with
\begin{equation}
\bar{\bfu}_{l,j}= \left\{
\begin{array}{ll}
(\bar{\bfC}_3)_{l,j} \left(\frac{1}{d}\sum\limits_{i=1}^{d}{\EigSq}_i+\lambda_C^l \Id{n}\right)^{-1}, & j=1,\\
\frac{1}{\lambda_C^l}\left[(\bar{\bfC}_3)_{l,j} - \frac{1}{d} \bar{\bfu}_{l,1} {\EigSq}_j\right], & j=2,\cdots,d.
\end{array}
\right.
\end{equation}
\end{theorem}

\begin{proof}
See Appendix \ref{app:theorem}.
\end{proof}

Note that $\bfu_{l,j} \in{\mathbb{R}^{1\times m}}$ denotes the $j$th block
of the $l$th band. Note also that the matrix
$\frac{1}{d}\sum\limits_{i=1}^{d}{\EigSq}_i+\lambda_C^l \Id{n}$ appearing
in the expression of $\bar{\bfu}_{l,1}$ is an $n \times n$ real
diagonal matrix whose inversion is trivial. The final estimator of $\bfX$ is
obtained as follows\footnote{It may happen that the diagonal matrix $\bfD$ does not have full rank (containing zeros in diagonal) or is ill-conditioned (having
very small numbers in diagonal), resulting from the property of blurring kernel. In this case, $\bfD^{-1}$ can be replaced by $\left(\bfD+\tau \Id{m}\right)^{-1}$
for regularization purpose, where $\tau$ is a small penalty parameter \cite{Lagendijk1990}.}
\begin{equation}
\hat{\bfX}={\bf HQ} \bar{\bfU} {\bfP} {\CirEig}^{-1} {\bfF}^H.
\end{equation}
Algo. \ref{Algo:Sylvester_Algo} summarizes the steps required to
calculate the estimated image $\hat{\bfX}$.

\IncMargin{1em}
\begin{algorithm}[h!]
\label{Algo:Sylvester_Algo}
\KwIn{$\bfY_{\mathrm{L}}$, $\bfY_{\mathrm{R}}$, $\NoiCovMat_{\mathrm{L}}$,
 $\NoiCovMat_{\mathrm{R}}$, $\bfL$, $\bfB$, $\bfS$, $\bfH$ }
 \tcpp{Circulant matrix decomposition: $\bfB= {\bfF \bfD \bfF}^H$}
$\bf{D} \leftarrow  \textrm{Dec} \left(\bfB\right)$\;
$\EigSq =\CirEig^{\ast}\CirEig$\;
 \tcpp{Calculate $\bfC_1$}
$\bfC_1 \leftarrow \left(\bfH^H \NoiCovMat_{\mathrm{R}}^{-1} \bfH\right)^{-1}\left({\left({\bfL}\bfH\right)}^H \NoiCovMat_{\textrm{L}}^{-1} {{\bfL}\bfH}\right)$\;
 \tcpp{Eigen-decomposition of $\bfC_1$: $\bfC_1=\bfQ{\bf{\Lambda}}_C\bfQ^{-1}$}
$\left({\bfQ,\bf{\Lambda}}_C\right) \leftarrow \textrm{EigDec}\left(\bfC_1\right)$\;
 \tcpp{Calculate $\bar{\bfC}_3$}
$\bar{\bfC}_3 \leftarrow \bfQ^{-1} \left({\bfH}^H \NoiCovMat_{\mathrm{R}}^{-1} {\bfH}\right)^{-1} ({\bfH}^H \NoiCovMat_{\mathrm{R}}^{-1} {\bfY}_\mathrm{R} \left(\bf BS \right)^H $\
$+ {\left({\bfL}\bfH\right)}^H \NoiCovMat_{\textrm{L}}^{-1} {\bfY}_\mathrm{L} ){\bf BF} {\bfP}^{-1}$\;
 \tcpp{Calculate $\bar{\bfU}$ block by block ($d$ blocks) and band by band ($\wtm_{\lambda}$ bands)}
\For{$l=1$ \KwTo $\wtm_{\lambda}$}{
\tcpp{Calculate the $1$st block in $l$th band}
$\bar{\bfu}_{l,1}= (\bar{\bfC}_3)_{l,1} \left(\frac{1}{d}\sum\limits_{i=1}^{d}{\EigSq}_i+\lambda_C^l \Id{n}\right)^{-1}$ \;
\tcpp{Calculate other blocks in $l$th band} 
\For{$j=2$ \KwTo $d$}{
$\bar{\bfu}_{l,j}= \frac{1}{\lambda_C^l}\left((\bar{\bfC}_3)_{l,j}- \frac{1}{d}\bar{\bfu}_{l,1} {\EigSq}_j\right)$\;
}
}
Set $\bfX= {\bf HQ}\bar{\bfU} {\bf PD}^{-1} {\bfF}^H$\;
\KwOut{${\MATima}$}
\caption{Fast Fusion of Multi-band Images}
\DecMargin{1em}
\end{algorithm}

\subsection{Complexity Analysis}
The most computationally expensive part of the proposed algorithm is the computation of matrices $\bfD$ and $\bar{\bfC}_3$
because of the FFT and iFFT operations. Using the notation $\bfC_4=\bfQ^{-1} \left({\bfH}^H \NoiCovMat_{\mathrm{R}}^{-1} {\bfH}\right)^{-1}$, the matrix $\bar{\bfC}_3$ can be rewritten
\begin{equation} \label{eq:complex}
\begin{array}{lll}
\bar{\bfC}_3  \hspace{-0.3cm} &=  \bfC_4 \left({\bfH}^H \NoiCovMat_{\mathrm{R}}^{-1} {\bfY}_\mathrm{R} \left(\bf BS \right)^H + {\left({\bfL}\bfH\right)}^H \NoiCovMat_{\textrm{L}}^{-1} {\bfY}_\mathrm{L}\right){\bf BF} {\bfP}^{-1}\\
&= \bfC_4 \left({\bfH}^H \NoiCovMat_{\mathrm{R}}^{-1} {\bfY}_\mathrm{R} {\bfS^H}{\bf FD}^{\ast} + \left({\bfL}\bfH\right)^H \NoiCovMat_{\textrm{L}}^{-1} {\bfY}_\mathrm{L} \bfF\right){\bfD} {\bfP}^{-1}.
\end{array}
\end{equation}
The most heavy step in computing \eqref{eq:complex} is the decomposition $\bfB=\bfF \bfD \bfF^H$ (or equivalently the FFT of the blurring kernel), which has a complexity of order $\mathcal{O}(n\log n)$. The calculations of ${\bfH}^H \NoiCovMat_{\mathrm{R}}^{-1} {\bfY}_\mathrm{R} {\bfS^H}{\bf FD}^{\ast}$ and $\left({\bfL}\bfH\right)^H \NoiCovMat_{\textrm{L}}^{-1} {\bfY}_\mathrm{L}\bfF$ require one FFT operation each. All the other computations are made in the frequency domain. Note that the multiplication by ${\bfD} {\bfP}^{-1}$ has a cost of $\mathcal{O}(n)$ operations as ${\bfD}$ is diagonal, and ${\bfP}^{-1}$ reduces to block shifting and addition. The left multiplication with $\bfQ^{-1} \left({\bfH}^H \NoiCovMat_{\mathrm{R}}^{-1} {\bfH}\right)^{-1}$ is of order
$\mathcal{O}(\wtm_{\lambda}^2 n)$. Thus, the calculation of $\bfC_3{\bf BF} {\bfP}^{-1}$ has a total complexity of order  $\mathcal{O}(n \cdot \mathrm{max}\left\{\log n,\wtm_{\lambda}^2\right\})$.


\section{Generalization to Bayesian estimators}
\label{sec:gen_Fast}
As mentioned in Section \ref{subsec:opt_ima}, if the matrix $\bfB \undS \bfB^H$ is singular or ill-conditioned (e.g., when
the number of MS bands is smaller than the dimension of the subspace spanned by the pixel vectors,
i.e., $n_{\lambda}<\wtm_{\lambda}$), a regularization or prior information $p\left(\bfU\right)$ has to be introduced to ensure the Sylvester matrix  equation \eqref{eq:sylv_3} has a unique solution. The resulting estimator $\bfU$ can then be interpreted as a Bayesian estimator. Combining the likelihood \eqref{eq:Likelihood_joint} and the prior $p\left(\bfU\right)$, Bayes theorem provides the posterior distribution of $\bfU$
\begin{equation*}
\label{eq:posterior_joint}
\begin{array}{lll}
  p\left(\bfU|\ALLobs\right) &\propto & p\left(\ALLobs|\bfU\right)p\left(\bfU\right) \\
  &\propto & p\left(\bfY_{\mathrm{L}}|\bfU\right)p\left(\bfY_{\mathrm{R}}|\bfU\right)p\left(\bfU\right)\\
\end{array}
\end{equation*}
where $\propto$ means ``proportional to'' and where we have used the independence between the observation vectors $\bfY_{\mathrm{L}}$ and $\bfY_{\mathrm{R}}$.


The mode of the posterior distribution $ p\left(\bfU|\ALLobs\right)$ is the so-called MAP estimator which
can be obtained by solving the following optimization problem 
\begin{equation}
\label{eq:OptWithU}
\argmin\limits_{\bfU} L(\bfU)
\end{equation}
where
\begin{equation}
\begin{array}{ll}
L(\bfU)=\frac{1}{2} \Tr \left({\left(\bfY_{\mathrm{R}}- \bf{HUBS}\right)^T\NoiCovMat_{\mathrm{R}}^{-1}\left(\bfY_{\mathrm{R}}- \bf{HUBS}\right)} \right) + \\
\quad \quad \frac{1}{2} \Tr \left(\left(\bfY_{\mathrm{L}} -{{\bfL}\bf HU}\right)^T \NoiCovMat_{\mathrm{L}}^{-1}\left(\bfY_{\mathrm{L}} -{{\bfL}\bf HU}\right)\right) -\log p(\bfU).
\end{array}
\end{equation}
Different Bayesian estimators corresponding to different choices of $p(\bfU)$ have been considered in the literature. These estimators are first recalled in the next sections. We will then show that the explicit solution of the SE derived in Section \ref{sec:FastFus} can be used to compute the MAP estimator of $\bfU$ for these prior distributions.


\subsection{Gaussian prior}
\label{subsec:GenGaussian}
Gaussian priors have been used widely in image processing \cite{Hardie1997,Eismann2004,Woods2006},
and can be interpreted as a Tikhonov regularization \cite{Tikhonov1977}. Assume that a matrix normal
distribution is assigned \apriori to the projected target image $\bfU$
\begin{equation}
\label{eq:prior_scene}
 p(\bfU)=\mathcal{MN}_{\wtm_{\lambda},n}(\bs{\mu}, \Covsub, \Id{n})
\end{equation}
where $\bs{\mu}$ and $\Covsub$ are the mean and covariance matrix of the matrix normal distribution. Note that the covariance matrix $\Covsub$ explores the correlations between HS band and controls the distance between $\bfU$ and its mean $\bs{\mu}$. Forcing the derivative of $L(\bfU)$ in \eqref{eq:OptWithU} to be zero leads to the following SE
\begin{equation}
\label{eq:zeroforce_Gaussian}
\bfC_1\bfU + \bfU \bfC_2 =\bfC_3
\end{equation}
where
\begin{equation}
\label{eq:newC1C2}
\begin{array}{ll}
\bfC_1=&\left(\bfH^H \NoiCovMat_{\mathrm{R}}^{-1} \bfH\right)^{-1}\left({\left({\bfL}\bfH\right)}^H \NoiCovMat_{\textrm{L}}^{-1} {{\bfL}\bfH}+\Covsub^{-1}\right)\\
\bfC_2=&{\bfB \undS \bfB}^H \\
\bfC_3=&\left(\bfH^H \NoiCovMat_{\mathrm{R}}^{-1} \bfH\right)^{-1} (\bfH^H \NoiCovMat_{\mathrm{R}}^{-1} {\bfY}_\mathrm{R} \left(\bf BS \right)^H +\\ &{\left({\bfL}\bfH\right)}^H\NoiCovMat_{\textrm{L}}^{-1} {\bfY}_\mathrm{L} +\Covsub^{-1}\bs{\mu}).
\end{array}
\end{equation}
The matrix $\bfC_1$ is positive definite as long as the covariance
matrix $\Covsub^{-1}$ is positive definite. Algo. \ref{Algo:Sylvester_Algo} can thus be adapted
to a matrix normal prior case by simply replacing $\bfC_1$ and $\bfC_3$ by their new expressions defined in \eqref{eq:newC1C2}.

\subsection{Non-Gaussian prior}
\label{subsec:GenNonGaussian}
The objective function $L(\bfU)$ in \eqref{eq:OptWithU} can be split into a data term $f(\bfU)$ corresponding to the likelihood
and a regularization term $\phi(\bfU)$ corresponding to the non-Gaussian prior in a Bayesian framework as
\begin{equation}
\begin{array}{lll}
L(\bfU) =f(\bfU) +\phi(\bfU)
\end{array}
\label{eq:neglog_post}
\end{equation}
where
\begin{equation*}
\begin{array}{rl}
f(\bfU) &= \frac{1}{2} \Tr \left({\left(\bfY_{\mathrm{R}}- \bf{HUBS}\right)^T\NoiCovMat_{\mathrm{R}}^{-1}\left(\bfY_{\mathrm{R}}- \bf{HUBS}\right)} \right)  \\
&+\frac{1}{2} \Tr \left(\left(\bfY_{\mathrm{L}} -{{\bfL}\bf HU}\right)^T \NoiCovMat_{\mathrm{L}}^{-1}\left(\bfY_{\mathrm{L}} -{{\bfL}\bf HU}\right)\right)
\end{array}
\end{equation*}
and
\begin{equation*}
\phi(\bfU)=-\log p\left(\bfU\right).
\end{equation*}
The optimization of \eqref{eq:neglog_post}  w.r.t. $\bfU$ can be solved efficiently
by using an ADMM that consists of two steps: 1) solving a surrogate optimization problem associated with a Gaussian prior and 2) applying a
proximity operator \cite{Combettes2011}. This strategy can be implemented in the image domain or in the frequency domain. The resulting algorithms, named as
SE-within-ADMM (SE-ADMM) methods, are described below.

\subsubsection{Solution in image domain}
Eq. \eqref{eq:neglog_post} can be rewritten as
\label{subsubsec:Format1}
\begin{equation*}
\begin{array}{lll}
L(\bfU,\bfV) = f(\bfU) +\phi(\bfV)
\end{array}
\textrm{s.t. } \bfU=\bfV.
\end{equation*}

The augmented Lagrangian associated with this problem is
\begin{equation}
L_{\mu}(\bfU,\bfV,\bs{\lambda}) = f(\bfU) +\phi(\bfV) + \bs{\lambda}^T ({\bf U-V}) + \frac{\mu}{2}\|\bfU-\bfV\|_F^2
\end{equation}
or equivalently
\begin{equation}
\label{eq:ADMMobj}
L_{\mu}(\bfU,\bfV,\bfW) = f(\bfU) +\phi(\bfV) + \frac{\mu}{2}\|\bfU-\bfV-\bfW\|_F^2
\end{equation}
where $\bfW$ is the scaled dual variable. This optimization problem can be solved by an ADMM as follows
\begin{equation*}
\begin{array}{rl}
(\bfU^{k+1},\bfV^{k+1})=&\argmin\limits_{\bfU,\bfV} f(\bfU) + \phi(\bfV)+ \\
&\frac{\mu}{2}\|\bfU-\bfV-\bfW^{k}\|_F^2\\
\vspace{-0.2cm}
 \bfW^{k+1}=&\bfW^{k}-(\bfU^{k+1}-\bfV^{k+1}).
\end{array}
\end{equation*}
More specifically, the updates of the derived ADMM algorithm are
\begin{equation}
\label{eq:ADMM_F1}
\begin{array}{rl}
\bfU^{k+1}&=\argmin\limits_{\bfU} f(\bfU)+\frac{\mu}{2}\|\bfU-\bfV^{k}-\bfW^{k}\|_F^2\\
\bfV^{k+1}&=\mathrm{prox}_{\phi,\mu}(\bfU^{k+1}-\bfW^{k})\\
\bfW^{k+1}&=\bfW^{k}-(\bfU^{k+1}-\bfV^{k+1}).
\end{array}
\end{equation}
\begin{itemize}
\item \textbf{Update $\bfU$}: Instead of using any iterative update method,
the optimization w.r.t. $\bfU$ can be solved analytically by using Algo. \ref{Algo:Sylvester_Algo} as for the
Gaussian prior investigated in Section \ref{subsec:GenGaussian}. For this, we can set $\bs{\mu}=\bfV^{k}+\bfW^{k}$ and $\Covsub^{-1}={\mu}\Id{\wtm_{\lambda}}$ in \eqref{eq:newC1C2}. However, the computational complexity of updating $\bfU$ in each iteration is $\mathcal{O}(n \log n)$ because of the FFT and iFFT steps required for computing $\bar{\bfC}_3$ and $\bfU$ from $\bar{\bfU}$.
\item \textbf{Update $\bfV$}: The update of $\bfV$ requires computing a proximity operator, which depends on the form of
$\phi(\bfV)$. When the regularizer $\phi(\bfV)$ is simple enough, the proximity operator can be evaluated analytically. For example, if $\phi(\bfV)\equiv \|\bfV\|_1$, then
$$
\mathrm{prox}_{\phi,\mu}(\bfU^{k+1}-\bfW^{k})=\mathrm{soft}\left(\bfU^{k+1}-\bfW^{k},\frac{1}{\mu}\right)
$$
where $\mathrm{soft}$ is the soft-thresholding function defined as
\begin{equation*}
\mathrm{soft}(g,\tau)=\mathrm{sign}(g)\max(|g|-\tau,0).
\end{equation*}
More examples of proximity computations can be found in \cite{Combettes2011}.
\item \textbf{Update $\bfW$}: The update of $\bfW$ is simply a matrix addition whose implementation has a small computational cost.
\end{itemize}

\subsubsection{Solution in frequency domain}
\label{subsubsec:Format2}
\newcommand{\hdcir}[1]{\mathcal{#1}}
Recalling that $\bfB = {\bf{FDF}}^H$, a less computationally expensive solution
is obtained by rewriting $L(\bfU)$ in \eqref{eq:neglog_post} as
\begin{equation*}
\begin{array}{lll}
L(\hdcir{U},\hdcir{V}) = f(\hdcir{U}) +\phi(\hdcir{V})
\end{array}
\textrm{s.t. } \hdcir{U} =\hdcir{V}
\end{equation*}
where $\hdcir{U}={\bf UF}$ is the Fourier transform of $\bfU$ and
\begin{equation*}
\begin{split}
f(\hdcir{U}) = \frac{1}{2} \Tr \left({(\bfY_{\mathrm{R}}- {\bfH} \hdcir{U} {\bf{DF}}^H \bfS)^T\NoiCovMat_{\mathrm{R}}^{-1}(\bfY_{\mathrm{R}}- {\bfH} \hdcir{U} {\bf{DF}}^H \bfS)} \right) \\
\quad \quad + \frac{1}{2} \Tr \left(\left(\bfY_{\mathrm{L}} -{{\bf LH}\hdcir{U} \bfF^H}\right)^T \NoiCovMat_{\mathrm{L}}^{-1}\left(\bfY_{\mathrm{L}} -{{\bf LH}\hdcir{U} \bfF^H}\right)\right)
\end{split}
\end{equation*}
and
$$
\phi(\hdcir{V})=-\log p\left(\hdcir{V}\right).
$$
Thus, the ADMM updates, defined in the image domain by \eqref{eq:ADMM_F1}, can be rewritten in the frequency domain as
\begin{equation}
\label{eq:ADMM_F2}
\begin{array}{rl}
\hdcir{U}^{k+1}&=\argmin\limits_{\hdcir{U}} f(\hdcir{U})+\frac{\mu}{2}\|\hdcir{U}-\hdcir{V}^{k}-\hdcir{W}^{k}\|_F^2\\
\hdcir{V}^{k+1}&=\mathrm{prox}_{\phi,\mu}(\hdcir{U}^{k+1}-\hdcir{W}^{k})\\
\hdcir{W}^{k+1}&=\hdcir{W}^{k}-(\hdcir{U}^{k+1}-\hdcir{V}^{k+1}).
\end{array}
\end{equation}
At the $(k+1)$th ADMM iteration, updating $\hdcir{U}$ can be efficiently conducted thanks to a SE solver similar to Algo. \ref{Algo:Sylvester_Algo}, where the matrix $\bar{\bfC}_3$ is defined by
\begin{equation}
\label{eq:C3_frequency}
\bar{\bfC}_3 = \bfC_s+\bfC_c \left(\hdcir{V}^{k}+\hdcir{W}^{k}\right){\bf DP}^{-1}
\end{equation}
with
\begin{equation*}
\begin{array}{ll}
\bfC_s &=\bfQ^{-1} \left({\bfH}^H \NoiCovMat_{\mathrm{R}}^{-1} {\bfH}\right)^{-1}({\bfH}^H \NoiCovMat_{\mathrm{R}}^{-1} {\bfY}_\mathrm{R} {\bfS^H}{\bf FD}^H \\
 &+  \left({\bfL}\bfH\right)^H \NoiCovMat_{\textrm{L}}^{-1} {\bfY}_\mathrm{L}\bfF ) {\bf DP}^{-1}\\
\bfC_c &=\bfQ^{-1} \left({\bfH}^H \NoiCovMat_{\mathrm{R}}^{-1} {\bfH}\right)^{-1}\Covsub^{-1}.
\end{array}
\end{equation*}
Note that the update of $\bar{\bfC}_3$ does not require any FFT computation since $\bfC_s$ and $\bfC_c$ can be
calculated once and are not updated in the ADMM iterations.

\subsection{Hierarchical Bayesian framework}
\label{subsec:GenHB}
When using a Gaussian prior, a hierarchical Bayesian framework can be constructed by introducing a hyperprior to the
hyperparameter vector $\hypervect=\left\{\bs{\mu},\Covsub\right\}$. Several priors have been investigated in the literature based on generalized Gaussian distributions, sparsity-promoted $\ell_1$ or $\ell_0$ regularizations,
$\ell_2$ smooth regularization, or TV regularization. Denoting as $p(\hypervect)$ the prior of $\hypervect$, the optimization w.r.t. $\bfU$ can be replaced by an optimization w.r.t. $\left(\bfU,\hypervect\right)$ as follows
\begin{equation*}
\begin{array}{rl}
\left(\bfU,\hypervect\right)&=\argmax\limits_{\bfU,\hypervect} p\left(\bfU,\hypervect|\ALLobs\right)\\
&=\argmax\limits_{\bfU,\hypervect} p\left(\bfY_{\mathrm{L}}|\bfU\right)p\left(\bfY_{\mathrm{R}}|\bfU\right)p\left(\bfU|\hypervect\right)p\left(\hypervect\right).\\
\end{array}
\end{equation*}
A standard way of solving this problem is to optimize alternatively between $\bfU$ and $\hypervect$
using the following updates
\begin{equation*}
\begin{array}{rl}
\bfU^{k+1}&=\argmax\limits_{\bfU} p\left(\bfY_{\mathrm{L}}|\bfU\right)p\left(\bfY_{\mathrm{R}}|\bfU\right)p\left(\bfU|\hypervect^{k}\right)\\
\hypervect^{k+1}&=\argmax\limits_{\hypervect} p\left(\bfU^{k+1}|\hypervect\right)p\left(\hypervect\right).
\end{array}
\end{equation*}
The update of $\bfU^{k+1}$ can be solved using Algo. \ref{Algo:Sylvester_Algo}
whereas the update of $\hypervect$ depends on the form of the hyperprior $p\left(\hypervect\right)$. The derived optimization method is named as a SE-within-block
coordinate descent (SE-BCD) method.

It is interesting to note that the strategy of Section \ref{subsec:GenNonGaussian} proposed to handle the case of a non-Gaussian prior can be interpreted as a special case of a hierarchical updating. Indeed, if we interpret $\bf V+d$ and $\frac{1}{\mu}\Id{\wtm_{\lambda}}$ in \eqref{eq:ADMMobj} as the mean $\bs{\mu}$ and covariance matrix $\Covsub$, the ADMM
update \eqref{eq:ADMM_F1} can be considered as the iterative updates of $\bfU$ and
$\bs{\mu}=\bf V+d$ with fixed $\Covsub=\frac{1}{\mu}\Id{\wtm_{\lambda}}$.

\section{Experimental results}
\label{sec:simu}

\newcommand{\figresultwidth}{0.30\textwidth}
\begin{figure*}[t!]
\centering
    \subfigure{
    \label{fig:subfig:Ref}
    \rotatebox{90}{\includegraphics[height=\figresultwidth]{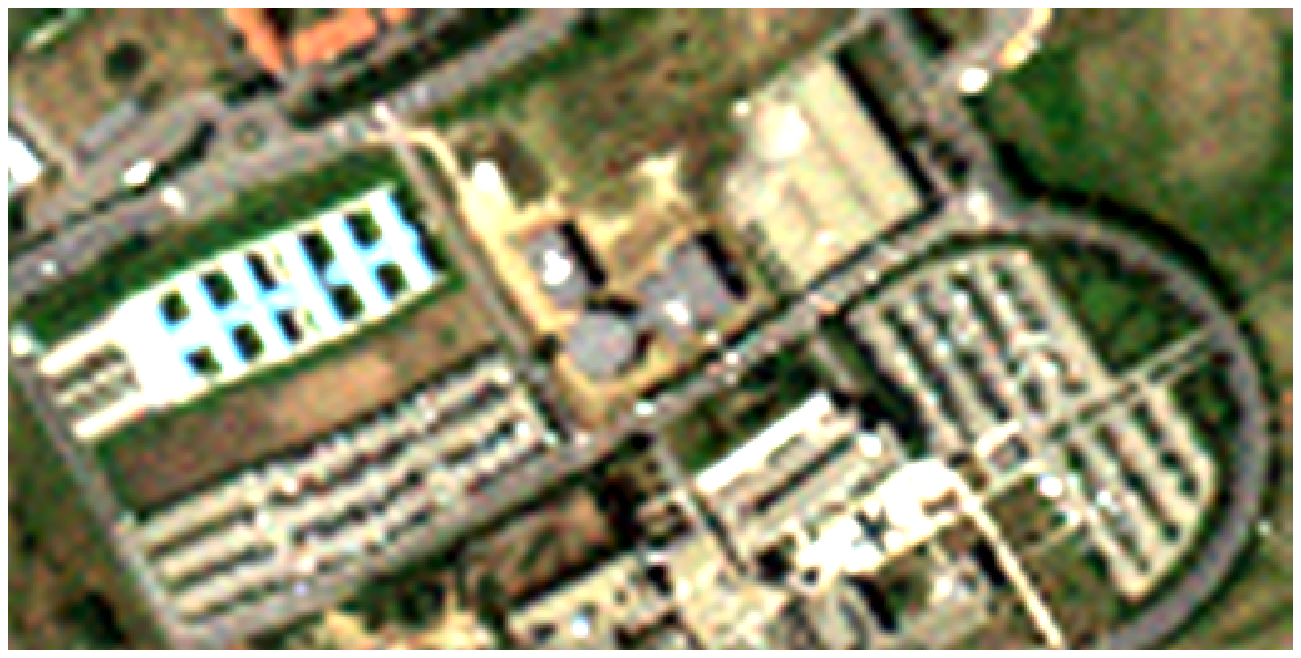}}}
    \subfigure{
    \label{fig:subfig:MS}
    \rotatebox{90}{\includegraphics[height=\figresultwidth]{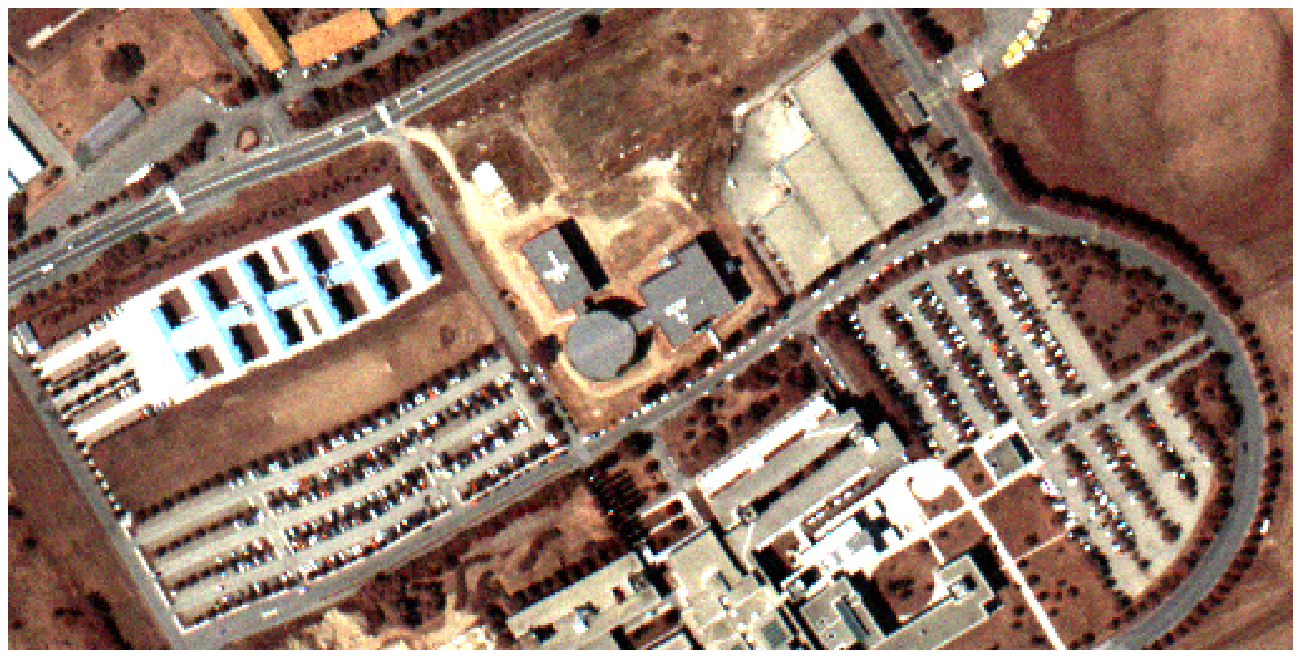}}}
    \subfigure{
    \label{fig:subfig:HS}
    \rotatebox{90}{\includegraphics[height=\figresultwidth]{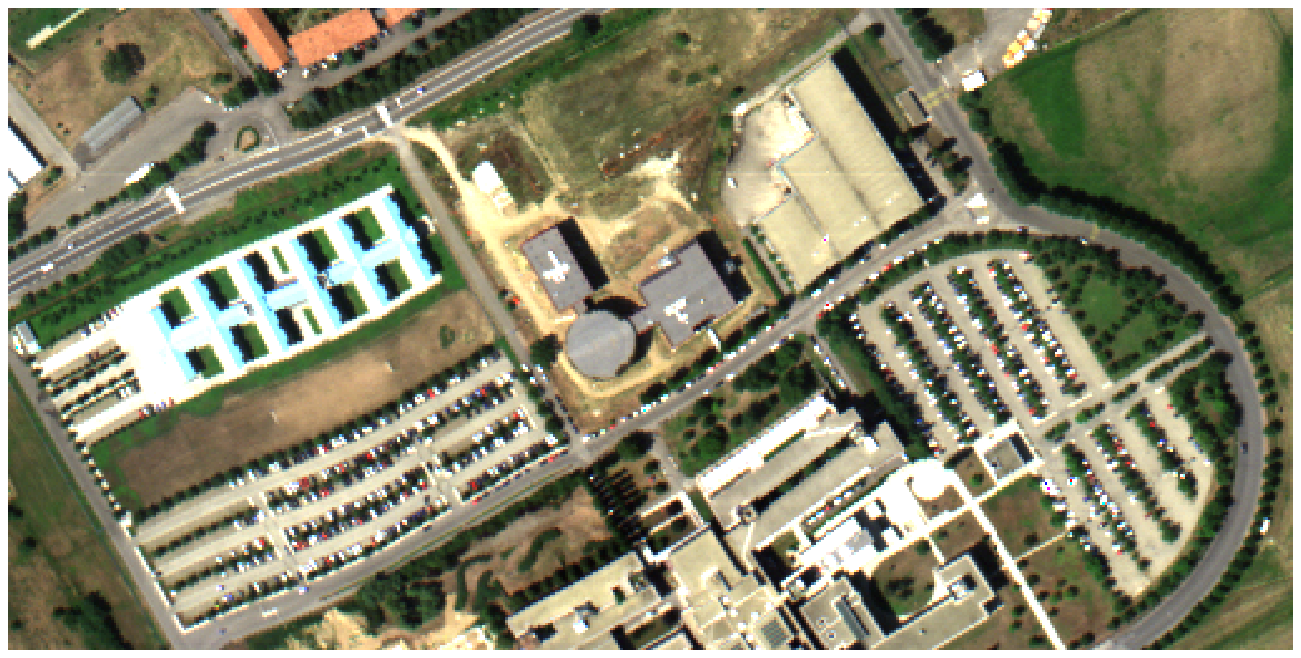}}}\\
    \subfigure{\centering
    \rotatebox{90}{\hspace{1cm} State-of-the-art methods}
    }
    \subfigure{
    \label{fig:subfig:ADMM}
    \includegraphics[width=0.17\textwidth]{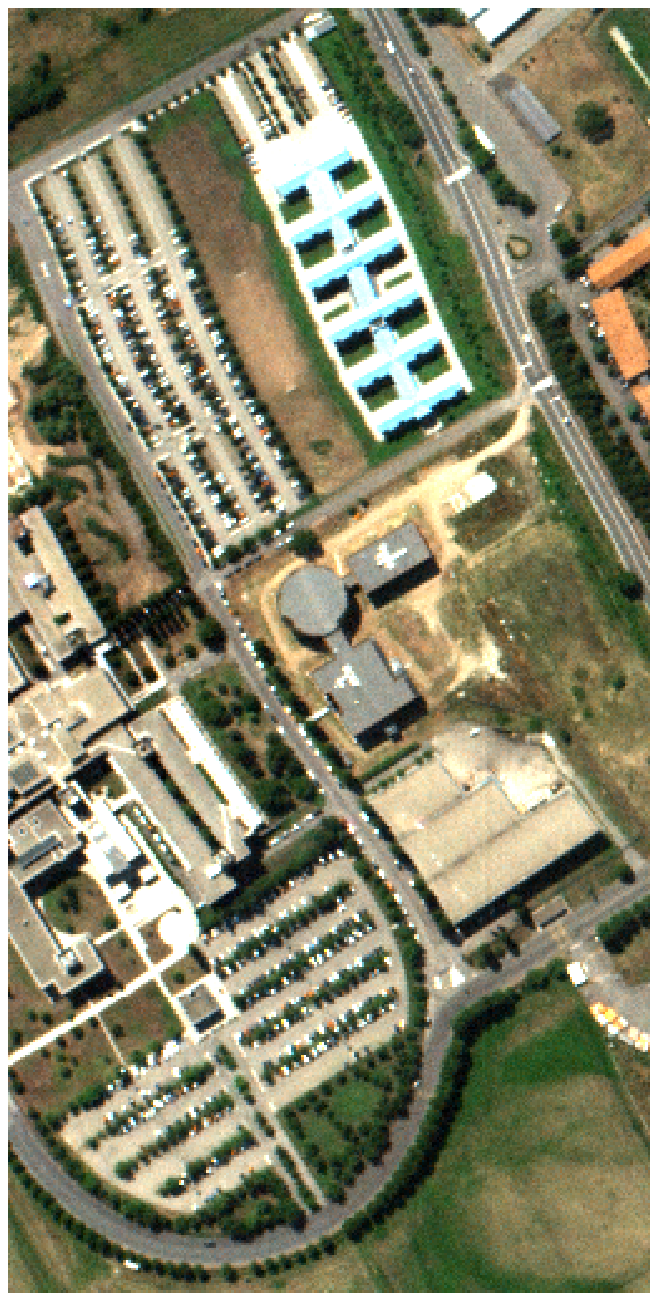}}
    \subfigure{
    \label{fig:subfig:ADMM}
    \includegraphics[width=0.17\textwidth]{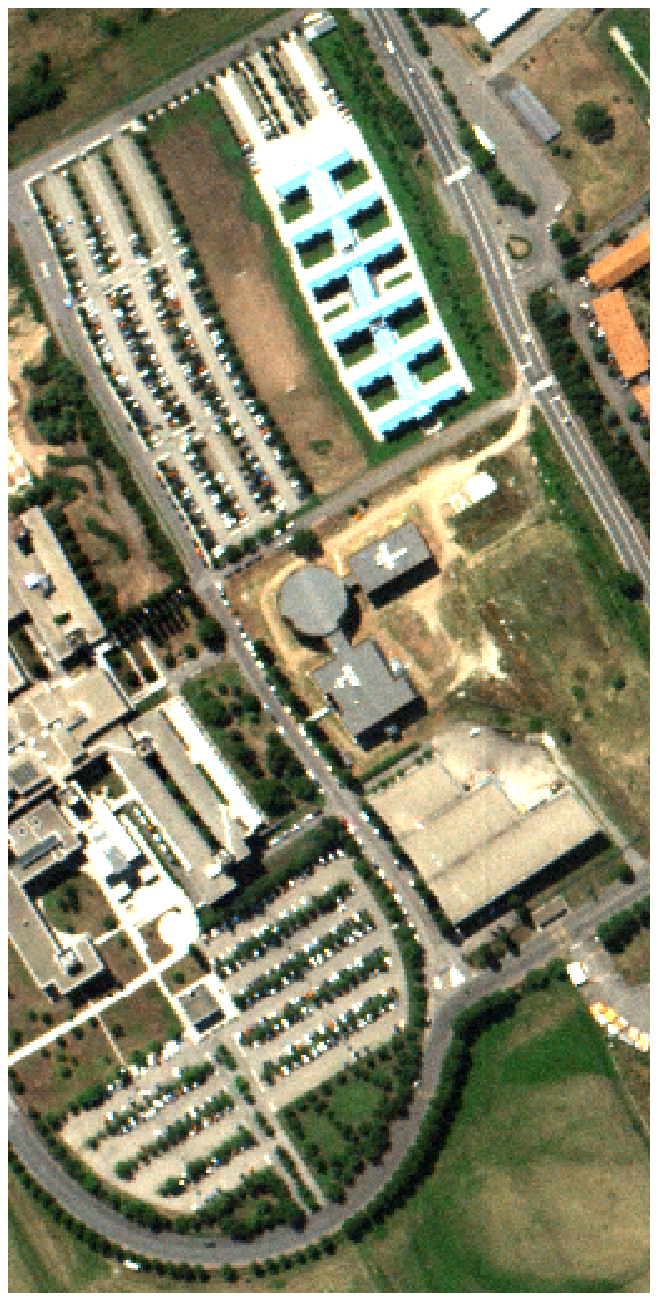}}
    \subfigure{
    \label{fig:subfig:ADMM}
    \includegraphics[width=0.17\textwidth]{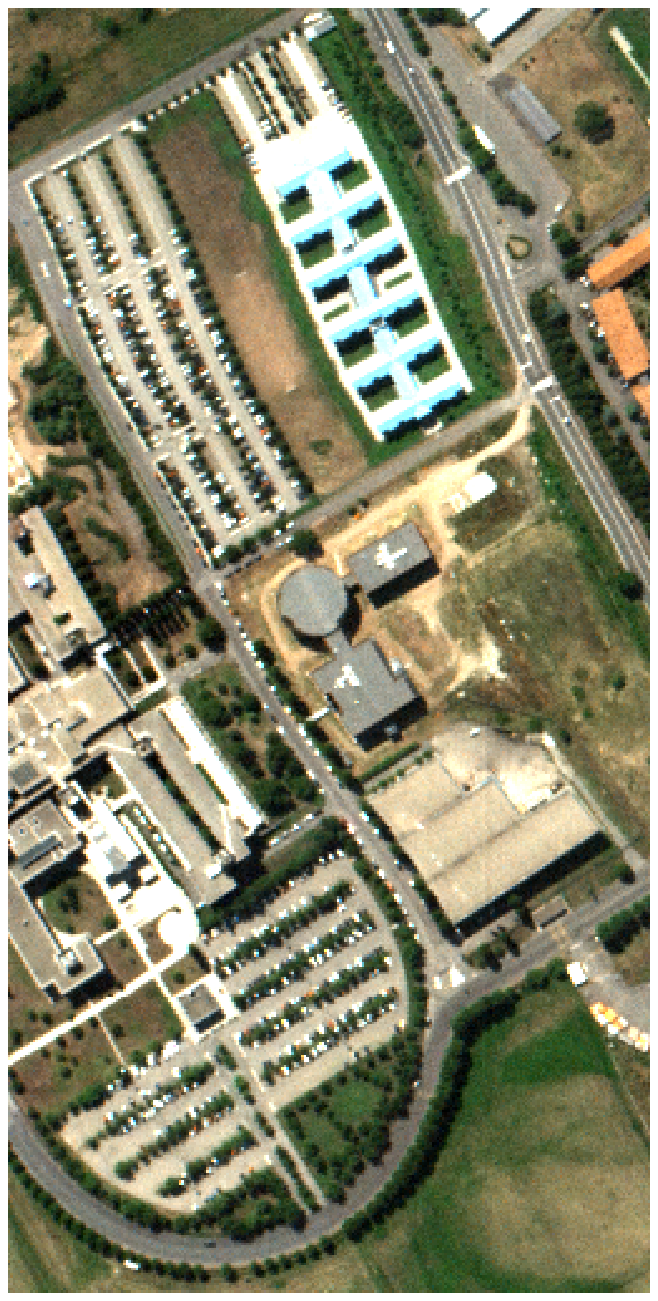}}
    \subfigure{
    \label{fig:subfig:ADMM}
    \includegraphics[width=0.17\textwidth]{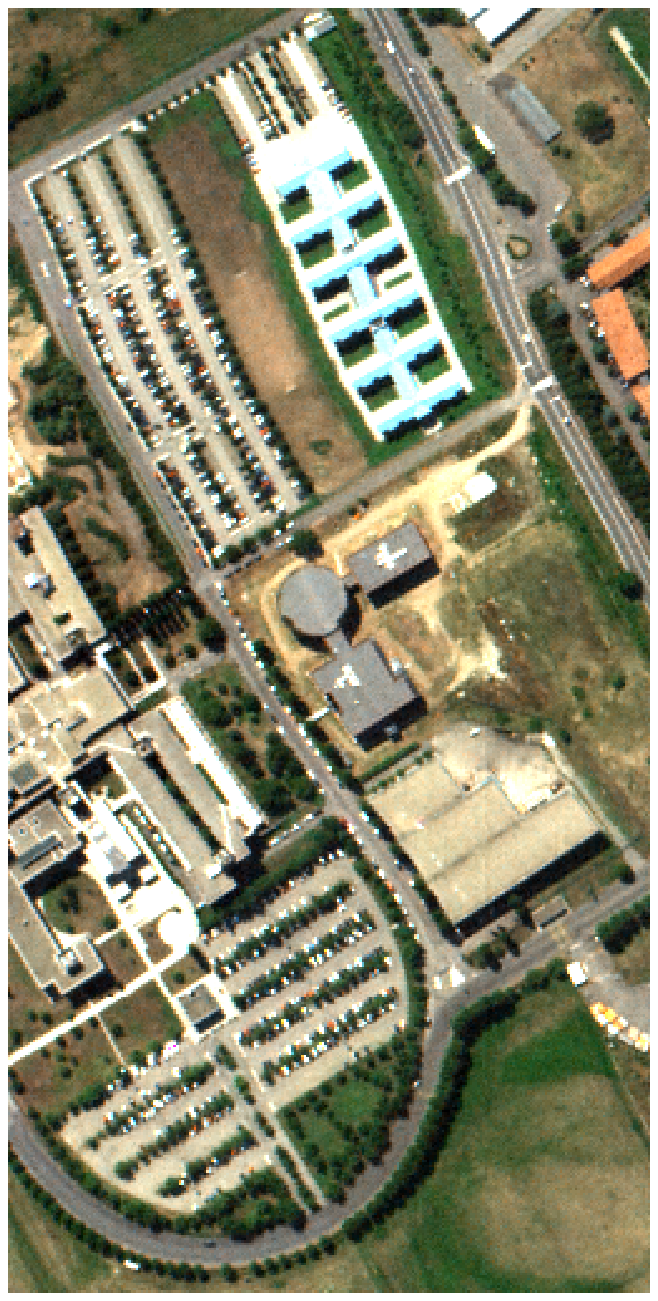}}\\
    \subfigure{
    \rotatebox{90}{\hspace{1.5cm} Fast fusion method}
    }
    \subfigure{
    \label{fig:subfig:SE}
    \includegraphics[width=0.17\textwidth]{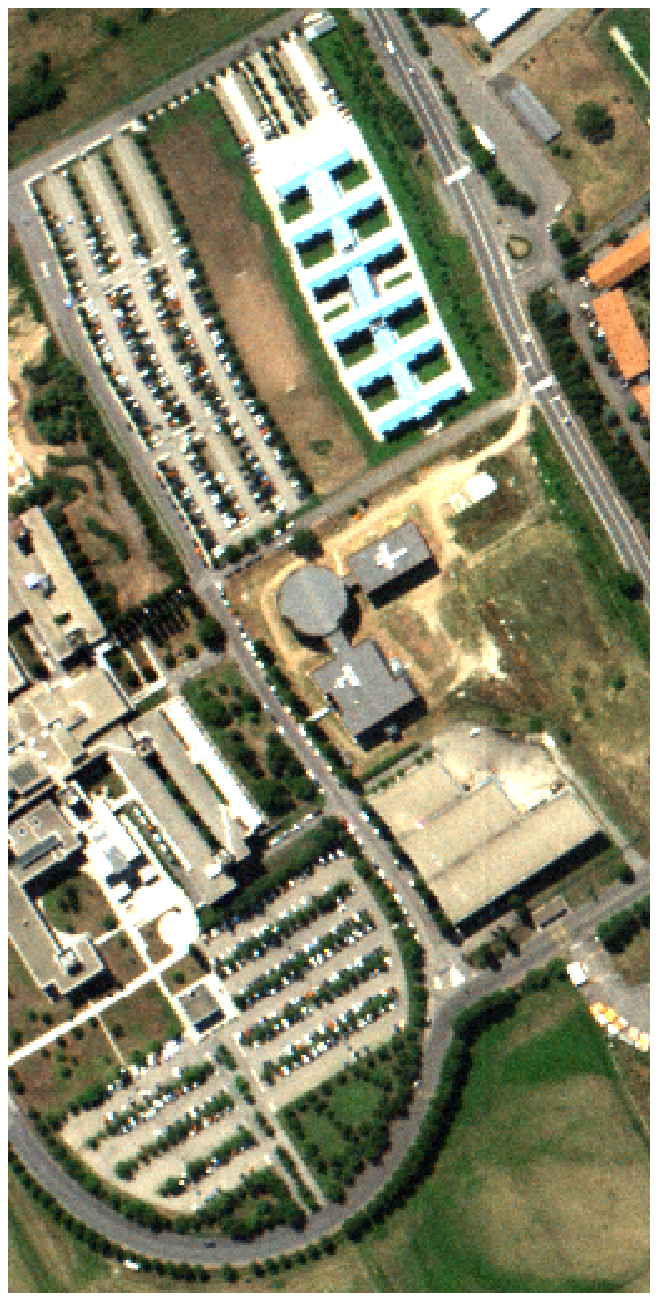}}
    \subfigure{
    \label{fig:subfig:SE}
    \includegraphics[width=0.17\textwidth]{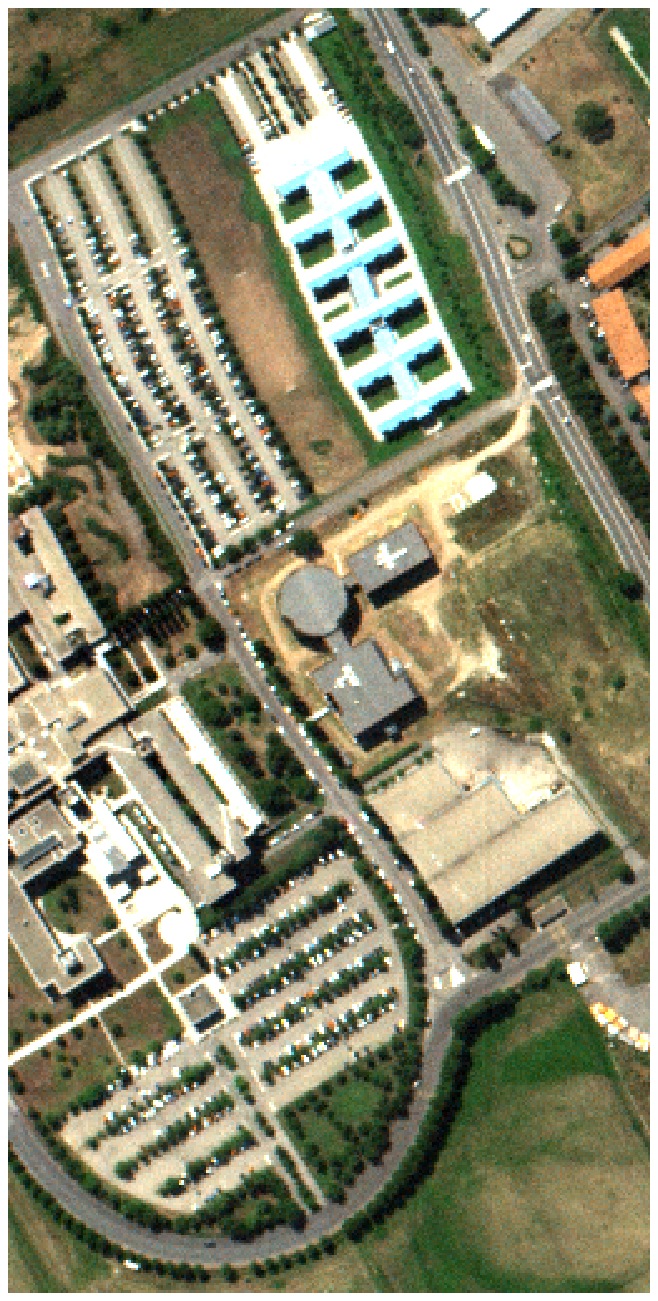}}
    \subfigure{
    \label{fig:subfig:SE}
    \includegraphics[width=0.17\textwidth]{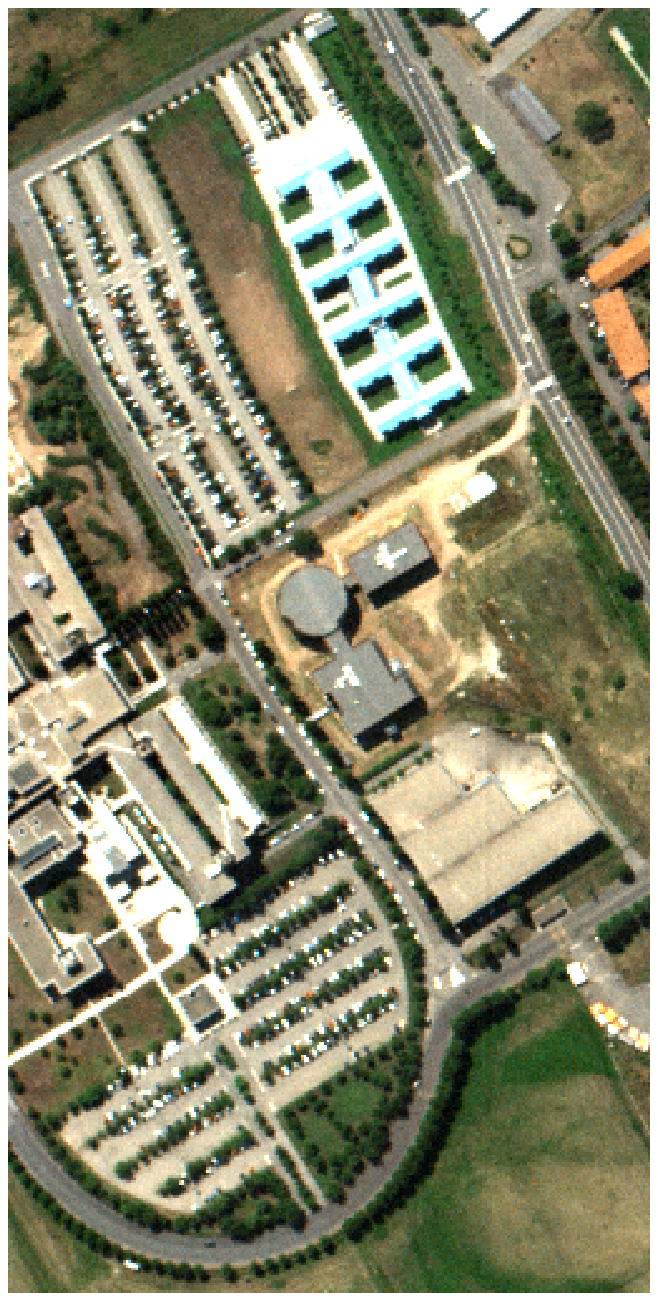}}
    \subfigure{
    \label{fig:subfig:SE}
    \includegraphics[width=0.17\textwidth]{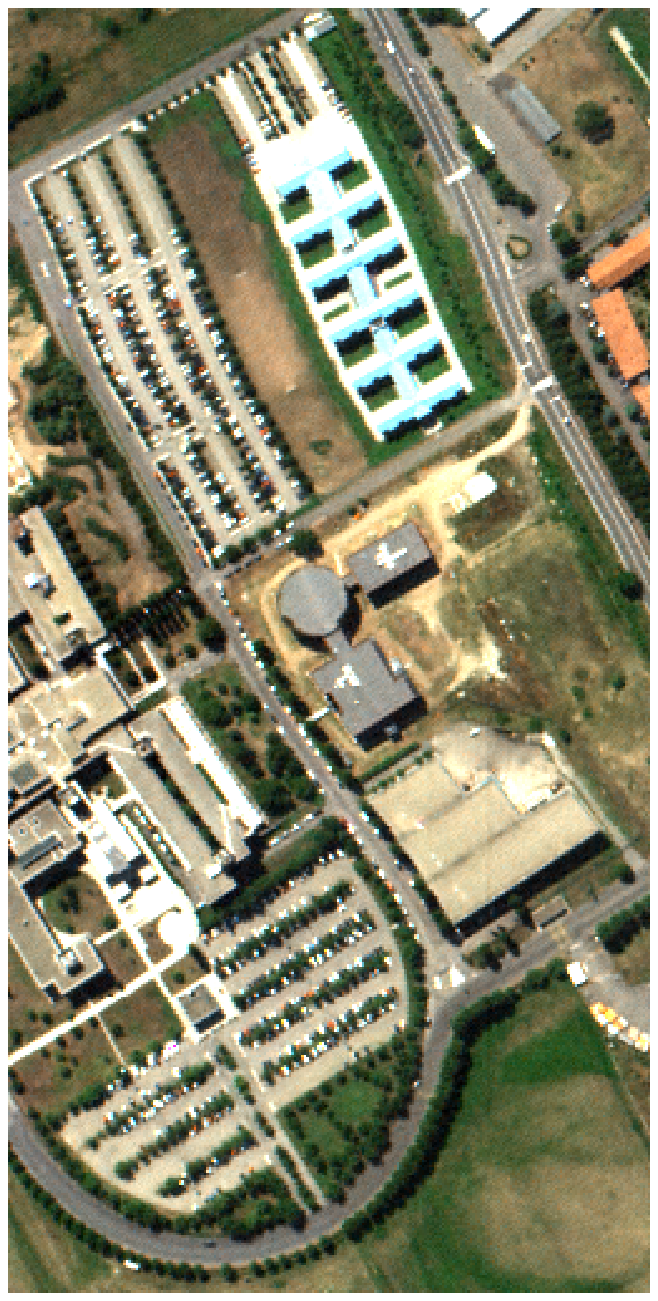}}
    \caption{HS+MS fusion results. Top: HS image ($1$st column), MS image ($2$nd column) and reference image ($3$rd column). Middle and bottom: state-of-the-art-methods and corresponding proposed fast fusion methods, respectively, with various regularizations: supervised naive Gaussian prior ($1$st column), unsupervised naive Gaussian prior ($2$nd column), sparse representation ($3$rd column) and TV ($4$th column).}
\label{fig:results}
\end{figure*}

This section applies the proposed fusion method to three kinds of priors that have been
investigated in \cite{Wei2015whispers}, \cite{Wei2015TGRS} and \cite{Simoes2015} for the fusion
of multi-band images. Note that these three methods require
to solve a minimization problem similar to \eqref{eq:OptWithU}.
All the algorithms have been implemented using MATLAB R2013A on a computer with
Intel(R) Core(TM) i7-2600 CPU@3.40GHz and 8GB RAM. The MATLAB codes and all the simulation 
results are available in the first author's homepage\footnote{\url{http://wei.perso.enseeiht.fr/}}.

\subsection{Fusion Quality Metrics}
\label{subsec:performance}
To evaluate the quality of the proposed fusion strategy, five
image quality measures have been investigated. Referring to \cite{Zhang2009,Wei2015TGRS},
we propose to use the restored signal to noise ratio (RSNR), the averaged spectral angle
mapper (SAM), the universal image quality index (UIQI), the relative dimensionless global
error in synthesis (ERGAS) and the degree of distortion (DD) as quantitative
measures. The RSNR is defined by the negative logarithm of the distance between the estimated
and reference images. The larger RSNR, the better the fusion. The definition of SAM, UIQI,
ERGAS and DD can be found in \cite{Wei2015TGRS}. The smaller SAM, ERGAS and DD, the better the fusion.
The larger UIQI, the better the fusion.

\subsection{Fusion of HS and MS images}
The reference image considered here as the high-spatial and high-spectral image is 
a $256 \times 128 \times 93$ HS image acquired over Pavia, Italy, by the reflective
optics system imaging spectrometer (ROSIS). This image was initially composed of $115$ bands
that have been reduced to $93$ bands after removing the water vapor absorption bands.
A composite color image of the scene of interest is shown in the top right of Fig. \ref{fig:results}.

Our objective is to reconstruct the high-spatial high-spectral image $\MATima$
from a low-spatial high-spectral HS image $\bfY_{\mathrm{R}}$ and a high-spatial
low-spectral MS image $\bfY_{\mathrm{L}}$. First, $\bfY_{\mathrm{R}}$ has been
generated by applying a $5 \times 5$ averaging filter and by down-sampling every
$4$ pixels in both vertical and horizontal directions for each band of the
reference image. Second, a $4$-band MS image $\bfY_{\mathrm{L}}$ has been obtained by
filtering $\MATima$ with the LANDSAT-like reflectance spectral responses \cite{Fleming2006}.
The HS and MS images are both contaminated by zero-mean additive Gaussian noises.
Our simulations have been conducted with
$\textrm{SNR}_{\mathrm{H},i}=35$dB
for the first $43$ bands of the HS image and $\textrm{SNR}_{\mathrm{H},i} =30$dB for the remaining $50$ bands, with
$$
\textrm{SNR}_{\mathrm{H},i}=10\log \left(\frac{\|\left(\MATima \bf BS\right)_i\|_F^2}{\noisevar{\mathrm{H},i}}\right).
$$
For the MS image
$$
\textrm{SNR}_{\mathrm{M},j}=10\log \left( \frac{\|\left({\bfL} \MATima\right)_j\|_F^2}{\noisevar{\mathrm{M},j}}\right) = 30\textrm{dB}
$$
for all spectral bands.

The observed HS and MS images are shown in the top left and middle figures of Fig. \ref{fig:results}.
Note that the HS image has been interpolated for better visualization and that the MS image has been
displayed using an arbitrary color composition. The subspace transformation matrix $\bfH$
has been defined as the PCA following the strategy of \cite{Wei2015TGRS}.

\subsubsection{Example 1: HS+MS fusion with a naive Gaussian prior}
\label{subsec:fuse_HS_MS}

We first consider a Bayesian fusion model initially proposed in \cite{Wei2013}. This method assumed a naive Gaussian prior for the target image, leading to an $\ell_2$-regularization of the fusion problem. The mean of this Gaussian prior was fixed to an interpolated HS image. The covariance matrix of the Gaussian prior can be fixed \emph{a priori} (supervised fusion) or estimated jointly with the unknown image within a hierarchical Bayesian method (unsupervised fusion). The estimator studied in \cite{Wei2013} was based on a hybrid Gibbs sampler generating samples distributed according to the posterior of interest. An ADMM step embedded in a BCD method (ADMM-BCD) was also proposed in \cite{Wei2015whispers} providing a significant computational cost reduction. This section compares the performance of this ADMM-BCD algorithm with the performances of the SE-based methods for these fusion problems.

\begin{table*}[t!]
\renewcommand{\arraystretch}{1.1}
\centering \caption{Performance of the fusion methods: RSNR (in dB), UIQI, SAM (in degree), ERGAS, DD (in $10^{-3}$) and time (in second).}
\vspace{0.1cm}
\begin{tabular}{|c|c|cccccc|}
\hline
Regularization & Methods & RSNR & UIQI & SAM  & ERGAS & DD & Time \\
\hline
\hline
{supervised }   &ADMM \cite{Wei2015whispers}&  29.295  & 0.9906 & 1.556 & 0.892 &7.121 &88.58\\
\cline{2-8}
naive Gaussian& Proposed SE  &  29.366 &  0.9908 & 1.552 & 0.880 &7.094 & 0.57\\
\hline
\hline
{unsupervised } &ADMM-BCD \cite{Wei2015whispers}&  29.078  & 0.9902 & 1.595 & 0.914 &7.279 &53.90\\
\cline{2-8}
naive Gaussian &Proposed SE-BCD &  29.076 &  0.9902 & 1.624 & 0.914&7.369 & 1.07\\
\hline
\hline
{sparse }  &ADMM-BCD \cite{Wei2015TGRS}&  29.575  & 0.9912 & 1.474 & 0.860 &6.833 &96.09\\
\cline{2-8}
representation&Proposed  SE-BCD&  29.600 &  0.9913 & 1.472 & 0.856 &6.820 &23.72\\
\hline
\hline
\multirow{2}{*}{TV} &ADMM \cite{Simoes2015}&  29.470  & 0.9911 & 1.503 & 0.861 &6.923 &138.98\\
\cline{2-8}
&Proposed SE-ADMM &  29.626 &  0.9914 & 1.478 & 0.846 &6.795&93.94\\
\hline
\end{tabular}
\label{tb:quality}
\end{table*}

For the supervised case, the explicit solution of the SE can be constructed directly following the
Gaussian prior-based generalization in Section \ref{subsec:GenGaussian}. Conversely, for the
unsupervised case, the generalized version denoted SE-BCD and described in Section \ref{subsec:GenHB}
is exploited, which requires embedding the closed-form solution into a BCD
algorithm (refer \cite{Wei2015whispers} for more details). The estimated images obtained
with the different algorithms are depicted in Fig. \ref{fig:results} and are visually very
similar. More quantitative results are reported in Table \ref{tb:quality} and confirm the similar performance of these methods in terms of the various fusion quality measures (RSNR, UIQI, SAM, ERGAS and DD).
However, the computational time of the proposed algorithm is reduced by a factor larger than $100$ (supervised) and $50$ (unsupervised) due to the existence of a closed-form solution for the Sylvester matrix equation.

\subsubsection{Example 2: HS+MS fusion with a sparse representation}
This section investigates a Bayesian fusion model based on a sparse representation introduced in \cite{Wei2015TGRS}. The
basic idea of this approach was to design a prior that results from the sparse decomposition of the target image on a set of dictionaries learned empirically. Some
parameters needed to be adjusted by the operator (regularization parameter, dictionaries and supports) whereas the other parameters (sparse codes) were jointly estimated with the target image. In \cite{Wei2015TGRS}, the MAP estimator associated with this model was reached using an optimization algorithm that consists of an ADMM step embedded in a BCD method (ADMM-BCD). Using the strategy proposed in Section \ref{subsec:GenHB}, this ADMM step can be avoid by exploiting the SE solution. Thus, the performance of the ADMM-BCD algorithm in \cite{Wei2015TGRS} is compared with the performance of the SE-BCD scheme as described in Section \ref{subsec:GenHB}. As shown in Fig. \ref{fig:results} and
Table \ref{tb:quality}, the performance of both algorithms is quite similar. However, the proposed solution exhibits a significant complexity reduction.

\subsubsection{Example 3: HS+MS Fusion with TV regularization}
The third experiment is based on a TV regularization (can be interpreted as a specific instance of a non-Gaussian prior)
studied in \cite{Simoes2015}. The regularization parameter of this method needs to be fixed by the user.
The ADMM-based method investigated in \cite{Simoes2015} requires to compute a TV-based proximity operator
(which increases the computational cost when compared to the previous algorithms). To solve this optimization problem, the frequency domain SE solution derived in Section \ref{subsubsec:Format2} can be embedded in an ADMM algorithm.
The fusion results obtained with the ADMM method of \cite{Simoes2015} and the proposed
SE-ADMM method are shown in Fig. \ref{fig:results} and are quite similar.
Table \ref{tb:quality} confirms this similarity more quantitatively by using
the quality measures introduced in Section \ref{subsec:performance}.
Note that the computational time obtained with the proposed explicit
fusion solution is reduced when compared to the ADMM method.
In order to complement this analysis, the convergence speeds of the
SE-ADMM algorithm and the ADMM method of \cite{Simoes2015}
are studied by analyzing the evolution of the objective function for the
two fusion solutions. Fig. \ref{fig:Conv_curve} shows that the SE-ADMM
algorithm converges faster at the starting phase and gives smoother convergence result.

\begin{figure}[h!]
\centering
\includegraphics[width=0.4\textwidth]{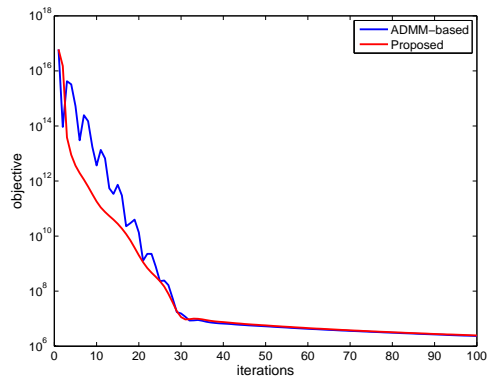}
\caption{Convergence speed of the ADMM \cite{Simoes2015} and the proposed SE-ADMM with
the TV-regularization.}
\label{fig:Conv_curve}
\end{figure}

\subsection{Hyperspectral Pansharpening}

\begin{figure*}[t!]
\centering
    \subfigure{
    \includegraphics[width=0.17\textwidth]{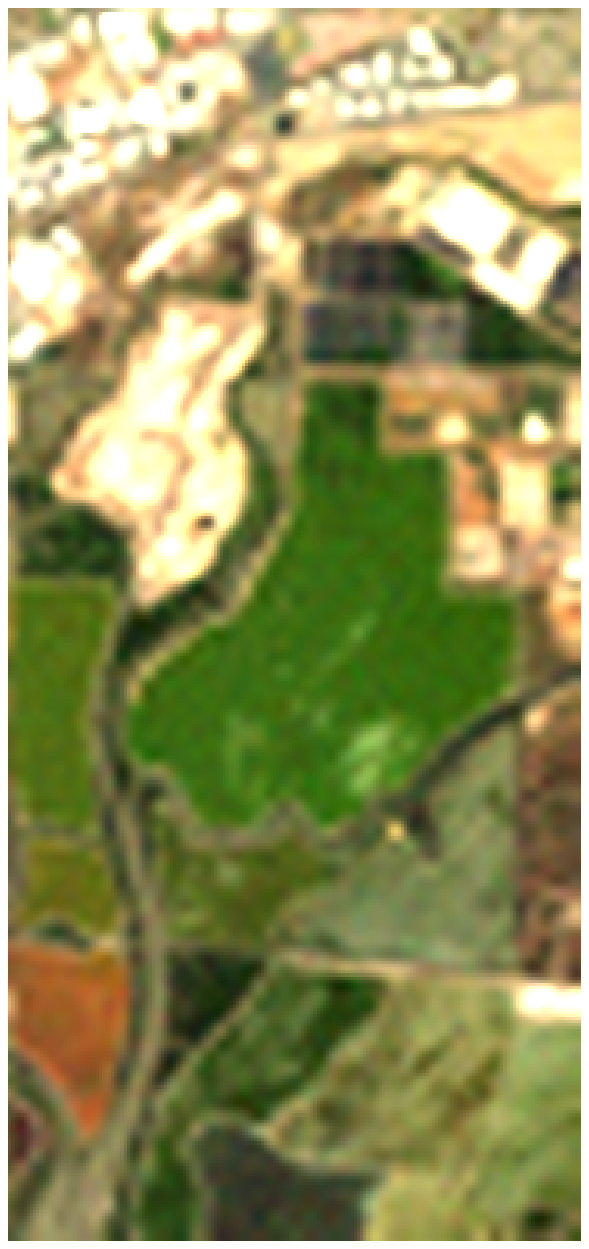}}
    \subfigure{
    \includegraphics[width=0.17\textwidth]{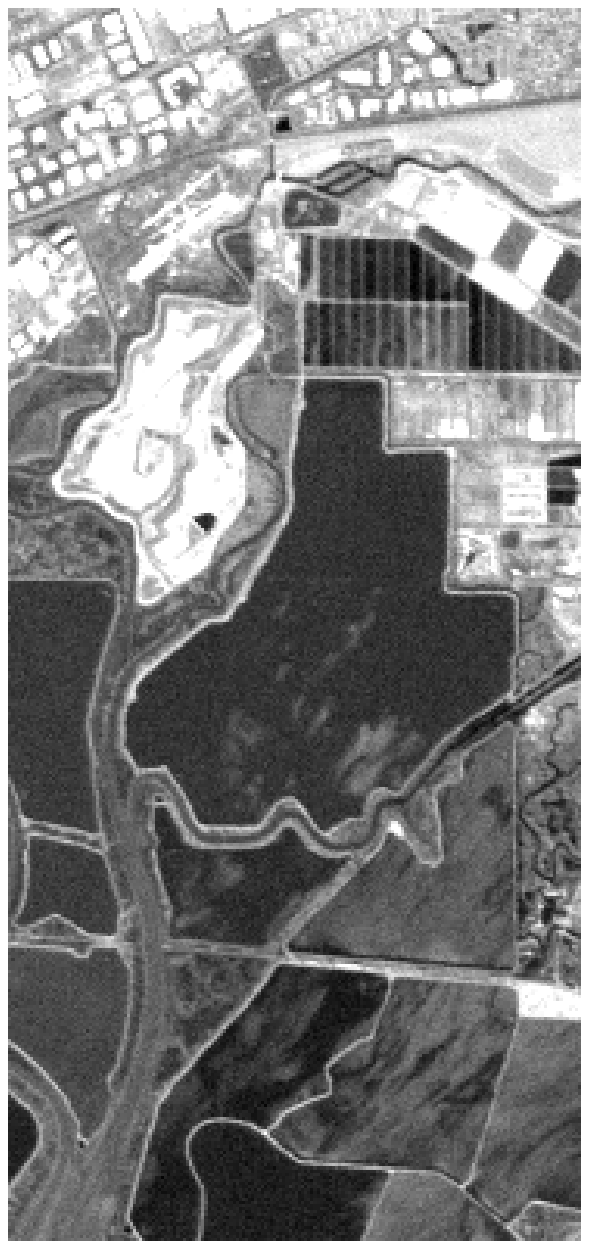}}
    \subfigure{
    \includegraphics[width=0.17\textwidth]{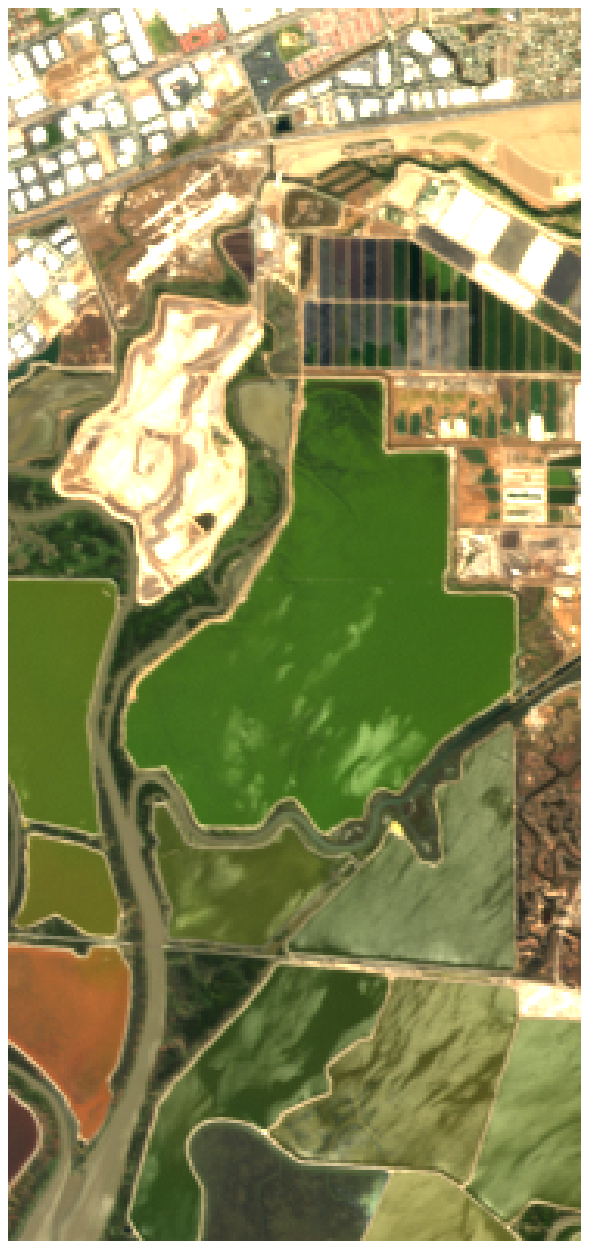}}
    \subfigure{
    \includegraphics[width=0.17\textwidth]{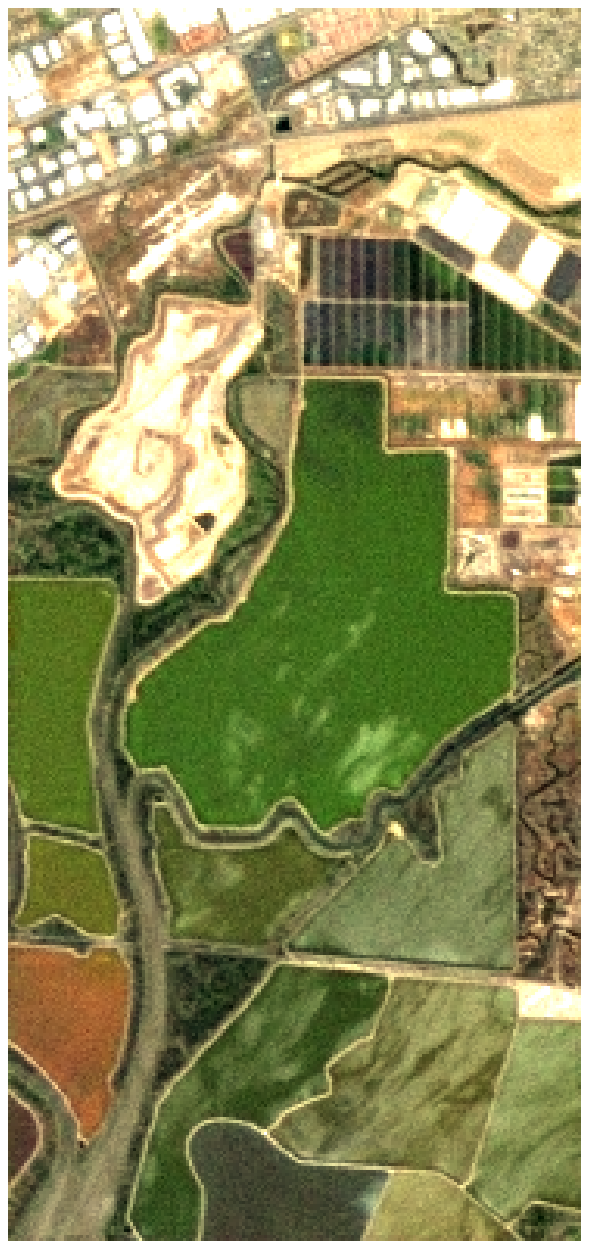}}
    \subfigure{
    \includegraphics[width=0.17\textwidth]{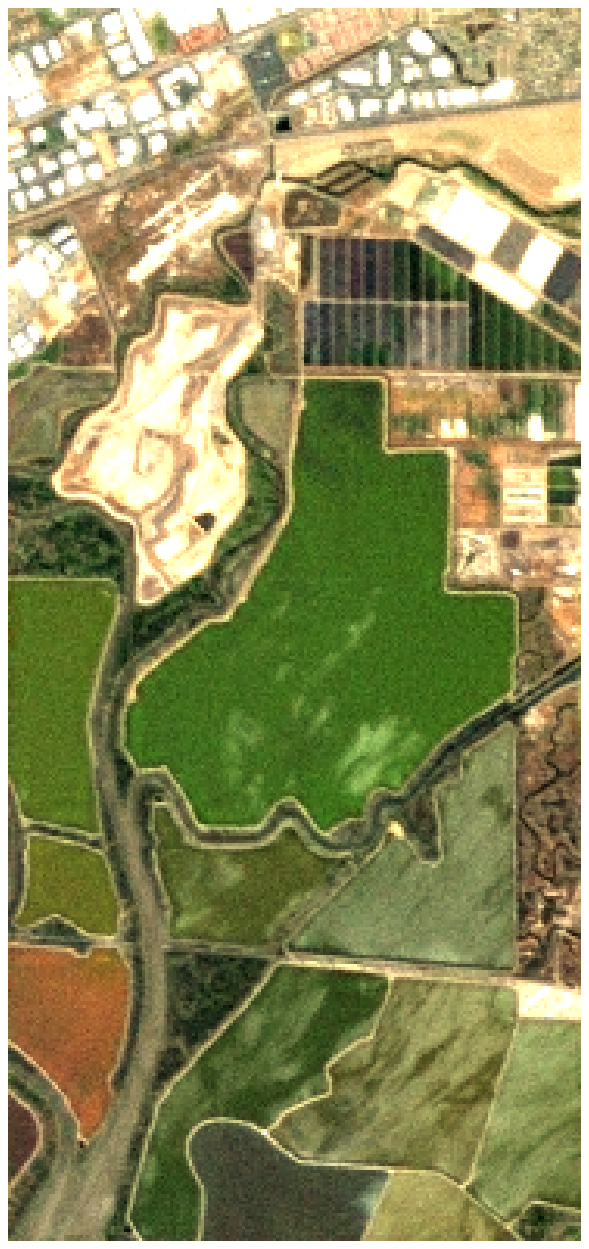}}
    \caption{Hyperspectral pansharpening results. $1$st column: HS image. $2$nd column: PAN image. $3$rd column: Reference image.
    $4$th column: ADMM \cite{Wei2015whispers}. $5$th column: Proposed method. }
\label{fig:HS_PAN_Moffet}
\end{figure*}

When $n_{\lambda}=1$, the fusion of HS and MS images reduces
to the HS pansharpening (HS+PAN) problem, which is the extension
of conventional pansharpening (MS+PAN) and has become an important and popular
application in the area of remote sensing \cite{Loncan2015}.
In order to show that the proposed method is also applicable to this problem,
we conducted HS and PAN image fusion for another HS
dataset. The reference image, considered here as the high-spatial
and high-spectral image, is an hyperspectral image acquired over
Moffett field, CA, in 1994 by the JPL/NASA airborne visible/infrared
imaging spectrometer (AVIRIS) \cite{Green1998imaging}. This image
was initially composed of $224$ bands that have been
reduced to $176$ bands after removing the water vapor absorption bands.
The HS image has been generated by applying a $5 \times 5$ averaging
filter on each band of the reference image. Besides, a PAN image
is obtained by successively averaging the adjacent bands in visible bands
(1$\sim$41 bands) according to realistic spectral responses.
In addition, the HS and PAN images have been both contaminated by
zero-mean additive Gaussian noises. The SNR of the HS image is $35$dB
for the first 126 bands and $30$dB for the last remaining bands.

A naive Gaussian prior\footnote{Due to space limitation, only the Gaussian prior 
of \cite{Wei2015whispers} is considered in this experiment. However, additional 
simulation results for other priors are available in the technical report \cite{Wei2015_TR_FaFusion}.}
is chosen to regularize this ill-posed inverse problem.
The SE based method is compared with the ADMM method of \cite{Wei2015whispers}
to solve the supervised pansharpening problem (i.e., with fixed hyperparameters). The results are displayed in
Fig. \ref{fig:HS_PAN_Moffet} whereas more quantitative results are reported in
Table \ref{tb:quality_PAN_Moffet}. Again, the proposed SE-based method provides
similar qualitative and quantitative fusion results with a significant computational
cost reduction.

\begin{table}[h!]
\renewcommand{\arraystretch}{1.1}
\setlength{\tabcolsep}{0.9mm}
\centering \caption{Performance of the Pansharpening methods: RSNR (in dB), UIQI, SAM (in degree), ERGAS, DD (in $10^{-2}$) and time (in second).}
\vspace{0.1cm}
\begin{tabular}{|c|cccccc|}
\hline
 Methods & RSNR & UIQI & SAM  & ERGAS & DD & Time \\
\hline
\hline
ADMM \cite{Wei2015whispers} &  18.616  & 0.9800 & 4.937 & 3.570 &1.306 &98.12\\
\hline
Proposed SE &  18.652 &  0.9802 & 4.939 & 3.555 &1.302 & 0.54\\
\hline
\end{tabular}
\label{tb:quality_PAN_Moffet}
\end{table}

\section{Conclusion}
\label{sec:concls}
This paper developed a fast multi-band image fusion method
based on an explicit solution of a Sylvester equation.
This method was applied to both the fusion of multispectral
and hyperspectral images and to the fusion of panchromatic
and hyperspectral images. Coupled with the alternating direction method
of multipliers and block coordinate descent, the proposed algorithm
can be easily generalized to compute Bayesian estimators for the fusion problem.
Besides, the analytically solution can be embedded in block coordinate
descent algorithm to solve the hierarchical
Bayesian estimation problem. Numerical
experiments showed that the proposed fast fusion method compares
competitively with the ADMM based methods, with the advantage
of reducing the computational complexity significantly.
Future work will consist of incorporating learning of the subspace
transform matrix $\bfH$ into the fusion scheme. Implementing the proposed
fusion scheme in real datasets will also be interesting.

\begin{appendices}
\section{Proof of Lemma \ref{lem:eig_sym}}
\label{app:eig_sym}
As $\bfA_1$ is symmetric (resp. Hermitian) positive definite, $\bfA_1$ can be
decomposed as $\bfA_1=\bfA_1^{\frac{1}{2}}\bfA_1^{\frac{1}{2}}$,
where $\bfA_1^{\frac{1}{2}}$ is also symmetric (resp. Hermitian) positive definite thus invertible.
Therefore, we have
\begin{equation}
\bfA_1 \bfA_2=  \bfA_1^{\frac{1}{2}} \left( \bfA_1^{\frac{1}{2}} \bfA_2  \bfA_1^{\frac{1}{2}} \right)
 \bfA_1^{-\frac{1}{2}} .
\end{equation}
As $\bfA_1^{\frac{1}{2}}$ and $ \bfA_2$ are both symmetric (resp. Hermitian) matrices,
$\bfA_1^{\frac{1}{2}} \bfA_2  \bfA_1^{\frac{1}{2}}$ is also a symmetric (resp. Hermitian) matrix that can be diagonalized. As a consequence, $\bfA_1\bfA_2$ is similar to a diagonalizable matrix, and thus it is diagonalizable.

Similarly, $\bfA_2$ can be written as $\bfA_2=\bfA_2^{\frac{1}{2}}\bfA_2^{\frac{1}{2}}$, where $\bfA_2^{\frac{1}{2}}$
is positive semi-definite. Thus, $\bfA_1^{\frac{1}{2}} \bfA_2  \bfA_1^{\frac{1}{2}}=\bfA_1^{\frac{1}{2}} \bfA_2^{\frac{1}{2}} \bfA_2^{\frac{1}{2}}  \bfA_1^{\frac{1}{2}}$
is positive semi-definite showing that all its eigenvalues are non-negative. As similar matrices share
equal similar eigenvalues, the eigenvalues of $\bfA_1\bfA_2$ are non-negative.

\section{Proof of Lemma \ref{lem:MMat}}
\label{app:MMat}
First, we introduce the following lemma.
\begin{lemma}
The following equality holds
\begin{equation*}
{\bfF}^H \undS {\bfF} = \frac{1}{d}\bfJ_{d} \otimes \Id{m}
\end{equation*}
where ${\bfF}$ and $\undS$ are defined as in Section \ref{subsec:opt_ima}, $\bfJ_d$ is the $d \times d$ matrix of
ones and $\Id{m}$ is the $m \times m$ identity matrix.
\label{lemm:3}
\end{lemma}
\begin{proof}
See Appendix \ref{app:lemma3proof}.
\end{proof}
According to Lemma \ref{lemm:3}, we have
\begin{equation}
\label{eq:block_prod}
{\bf{F}}^H \undS\bfF\EigSq = \frac{1}{d} \left(\bfJ_{d} \otimes \Id{m}\right)\EigSq
=\frac{1}{d}\left[
\begin{array}{cccc}
\EigSq_1 & \EigSq_2 & \cdots & \EigSq_d \\
\vdots & \vdots & \ddots & \vdots \\
\EigSq_1 & \EigSq_2 & \cdots & \EigSq_d \\
\end{array}
\right]
\end{equation}
Thus, multiplying \eqref{eq:block_prod} by $\bfP$ on the left side and by $\bfP^{-1}$
on the right side leads to
\begin{equation*}
\begin{array}{ll}
&\bfM = \bfP \left({\bfF}^H \undS {\bfF}\EigSq \right) \bfP^{-1} \\
& = \frac{1}{d}\left[
\begin{array}{ccccc}
{\EigSq}_i &{\EigSq}_2 &\cdots &{\EigSq}_d\\
\bs{0}& \bs{0}& \cdots &\bs{0}\\
\vdots&\vdots &\ddots &\vdots\\
\bs{0}& \bs{0}& \cdots &\bs{0}
\end{array}
\right] \bfP^{-1}\\
& = \frac{1}{d}\left[
\begin{array}{ccccc}
\sum\limits_{i=1}^{d}{\EigSq}_i &{\EigSq}_2 &\cdots &{\EigSq}_d\\
\bs{0}& \bs{0}& \cdots &\bs{0}\\
\vdots&\vdots &\ddots &\vdots\\
\bs{0}& \bs{0}& \cdots &\bs{0}
\end{array}
\right]
\end{array}
\label{eq:temp1}
\end{equation*}

\section{Proof of Theorem \ref{the:Ubar}}
\label{app:theorem}
Substituting \eqref{eq:MMat} and \eqref{eq:U_vector} into \eqref{eq:sylv_3} leads to \eqref{eq:element_wise}, where
\begin{equation}
\bar{\bfC}_3=\left[
\begin{array}{ccccc}
(\bar{\bfC}_3)_{1,1} & (\bar{\bfC}_3)_{1,2}&\cdots &(\bar{\bfC}_3)_{1,d}\\
(\bar{\bfC}_3)_{2,1} & (\bar{\bfC}_3)_{2,2}&\cdots &(\bar{\bfC}_3)_{2,d}\\
\vdots    & \vdots     & \ddots & \vdots \\
(\bar{\bfC}_3)_{d,1} & (\bar{\bfC}_3)_{d,2}&\cdots &(\bar{\bfC}_3)_{d,d}
\end{array}
\right].
\end{equation}
\begin{figure*}[!ht]
\begin{equation}
\label{eq:element_wise}
\left[
\begin{array}{ccccc}
\bar{\bfu}_{1,1} \left(\frac{1}{d}\sum\limits_{i=1}^{d}{\EigSq}_i+\lambda_C^1 \Id{n}\right)& \lambda_C^1\bar{\bfu}_{1,2}+\frac{1}{d}\bar{\bfu}_{1,1}{\EigSq}_2&\cdots &\lambda_C^1 \bar{\bfu}_{1,d}+\frac{1}{d}\bar{\bfu}_{1,1}{\EigSq}_d\\
\bar{\bfu}_{2,1} \left(\frac{1}{d}\sum\limits_{i=1}^{d}{\EigSq}_i+\lambda_C^2 \Id{n}\right)& \lambda_C^2\bar{\bfu}_{2,2}+\frac{1}{d}\bar{\bfu}_{2,1}{\EigSq}_2 &\cdots &\lambda_C^2 \bar{\bfu}_{2,d}+\frac{1}{d}\bar{\bfu}_{2,1}{\EigSq}_d\\
\vdots    & \vdots     & \ddots & \vdots\\
\bar{\bfu}_{\wtm_{\lambda},1} \left(\frac{1}{d}\sum\limits_{i=1}^{d}{\EigSq}_i+\lambda_C^{\wtm_{\lambda}} \Id{n}\right)
& \lambda_C^{\wtm_{\lambda}}\bar{\bfu}_{\wtm_{\lambda},2}+\frac{1}{d}\bar{\bfu}_{\wtm_{\lambda},1}{\EigSq}_2 & \cdots & \lambda_C^{\wtm_{\lambda}}\bar{\bfu}_{\wtm_{\lambda},d}+\frac{1}{d}\bar{\bfu}_{\wtm_{\lambda},1}{\EigSq}_d
\end{array}
\right]
= \bar{\bfC}_3
\end{equation}
\end{figure*}
Identifying the first (block) columns of \eqref{eq:element_wise}
allows us to compute the element $\bar{\bfu}_{1,1}$ for $l=1,...,d$
as follows
\begin{equation*}
\bar{\bfu}_{l,1} = (\bar{\bfC}_3)_{l,1} \left(\frac{1}{d}\sum\limits_{i=1}^{d}{\EigSq}_i+\lambda_C^l \Id{n}\right)^{-1}
\end{equation*}
for $l=1,\cdots,\wtm_{\lambda}$. Using the values of $\bar{\bfu}_{l,1}$ determined above,
it is easy to obtain $\bar{\bfu}_{l,2},\cdots,\bar{\bfu}_{l,d}$ as
\begin{equation*}
\bar{\bfu}_{l,j}=  \frac{1}{\lambda_C^l}\left[(\bar{\bfC}_3)_{l,j} - \frac{1}{d} \bar{\bfu}_{l,1} {\EigSq}_j\right]
\end{equation*}
for $l=1,\cdots,\wtm_{\lambda}$ and $j=2,\cdots,d$.

\section{Proof of Lemma \ref{lemm:3}}
\label{app:lemma3proof}
The $n$ dimensional DFT matrix $\bfF$ can be written explicitly as follows
\begin{equation*}
\bfF = \frac{1}{\sqrt{n}} \begin{bmatrix}
1&1&1&1&\cdots &1 \\
1&\omega&\omega^2&\omega^3&\cdots&\omega^{n-1} \\
1&\omega^2&\omega^4&\omega^6&\cdots&\omega^{2(n-1)}\\ 1&\omega^3&\omega^6&\omega^9&\cdots&\omega^{3(n-1)}\\
\vdots&\vdots&\vdots&\vdots&\ddots&\vdots\\
1&\omega^{n-1}&\omega^{2(n-1)}&\omega^{3(n-1)}&\cdots&\omega^{(n-1)(n-1)}\\
\end{bmatrix}
\end{equation*}
where $\omega = e^{-\frac{2\pi i}{n}}$ is a primitive $n$th root of unity in which $i=\sqrt{-1}$.
The matrix $\undS$ can also be written as follows
\begin{equation*}
\undS = \bfE_1+\bfE_{1+d}+\cdots+\bfE_{1+(m-1)d}
\end{equation*}
where $\bfE_i \in \mathbb{R}^{n\times n}$ is a matrix containing
only one non-zero element equal to 1 located at the $i$th row and
$i$th column as follows
\begin{equation*}
\bfE_i=
 \begin{bmatrix}
0&\cdots&0&\cdots &0 \\
\vdots & \ddots & \vdots & \ddots & \vdots \\
0&\cdots&1&\cdots&0 \\
\vdots & \ddots & \vdots & \ddots & \vdots\\
0&\cdots&0&\cdots &0 \\
\end{bmatrix}.
\end{equation*}
It is obvious that $\bfE_i$ is an idempotent matrix, i.e., $\bfE_i=\bfE_i^2$.
Thus, we have
\begin{equation*}
\bfF^H\bfE_i\bfF = \left(\bfE_i\bfF\right)^H \bfE_i\bfF
=\left[ \bs{0}^T \cdots \bff_i^H \cdots \bs{0}^T \right] \left[
\begin{array}{cc}
\bs{0}\\
\vdots\\
\bff_i\\
\vdots\\
\bs{0}
\end{array}
\right]= \bff_i^H\bff_i
\end{equation*}
where $\bff_i= \frac{1}{\sqrt{n}}\left[ 1 \quad \omega^{i-1} \quad \omega^{2(i-1)} \quad \omega^{3(i-1)}  \cdots \omega^{(n-1)(i-1)} \right]$ is the $i$th row of the matrix $\bfF$ and $\bs{0} \in \mathbb{R}^{1\times n}$ is
the zero vector of dimension $1 \times n$.
Straightforward computations lead to
\begin{equation*}
\bff_i^H\bff_i=\frac{1}{n}
 \begin{bmatrix}
1&\omega^{i-1}&\cdots &\omega^{(i-1)(n-1)} \\
\omega^{-(i-1)}&1&\cdots &\omega^{(i-1)(n-2)}\\
\vdots & \vdots & \ddots & \vdots \\
\omega^{-(i-1)(n-1)}&\omega^{-(i-1)(n-2)}&\cdots &1
\end{bmatrix}.
\end{equation*}
Using the $\omega$'s property $\sum\limits_{i=1}^{n} \omega^{i}=0$ and $n=md$ leads to
\begin{equation*}
\begin{array}{ll}
\bff_1^H\bff_1+\bff_{1+d}^H\bff_{1+d}+\cdots \bff_{1+(m-1)d}^H\bff_{1+(m-1)d}\\
=\frac{1}{n}
\begin{bmatrix}
\begin{bmatrix}
m & 0 & \cdots & 0  \\
0 & m & \cdots & 0 \\
\vdots & \vdots & \ddots & \vdots\\
0 & 0 & \cdots & m \\
\end{bmatrix}& \cdots &\begin{bmatrix}
m & 0 & \cdots & 0  \\
0 & m & \cdots & 0 \\
\vdots & \vdots & \ddots & \vdots\\
0 & 0 & \cdots & m \\
\end{bmatrix} \\
\vdots& \ddots & \vdots &\\
\begin{bmatrix}
m & 0 & \cdots & 0  \\
0 & m & \cdots & 0 \\
\vdots & \vdots & \ddots & \vdots\\
0 & 0 & \cdots & m \\
\end{bmatrix}& \cdots & \begin{bmatrix}
m & 0 & \cdots & 0  \\
0 & m & \cdots & 0 \\
\vdots  & \ddots & \vdots & \vdots\\
0 & 0 & \cdots & m \\
\end{bmatrix}
\end{bmatrix}\\
=\frac{1}{d}
\begin{bmatrix}
\Id{m} & \cdots &\Id{m} \\
\vdots & \ddots &\vdots \\
\Id{m} & \cdots &\Id{m} \\
\end{bmatrix}\\
=\frac{1}{d} \bfJ_{d} \otimes \Id{m}.
\end{array}
\end{equation*}
\end{appendices}
\bibliographystyle{ieeetran}
\bibliography{strings_all_ref,Sylvesterbib}
\end{document}